\documentclass{article} 
\usepackage{arxiv,times}

\usepackage{amsmath,amsfonts,bm}

\def\eqref#1{equation~\ref{#1}}

\def\1{\bm{1}}

\DeclareMathAlphabet{\mathsfit}{\encodingdefault}{\sfdefault}{m}{sl}
\SetMathAlphabet{\mathsfit}{bold}{\encodingdefault}{\sfdefault}{bx}{n}

\newcommand{\E}{\mathbb{E}}

\newcommand{\R}{\mathbb{R}}

\DeclareMathOperator*{\argmax}{arg\,max}
\DeclareMathOperator*{\argmin}{arg\,min}

\usepackage{graphicx}
\usepackage{hyperref}
\usepackage{url}
\usepackage{booktabs}
\usepackage{amsfonts}
\usepackage{nicefrac}
\usepackage{microtype}
\usepackage{xcolor}
\usepackage{algorithm}
\usepackage{algpseudocode}
\usepackage{wrapfig}
\usepackage{amsmath,amssymb,amsthm,mathtools,bm}
\usepackage{enumitem}
\usepackage{multirow}
\usepackage{array}
\usepackage{comment}
\usepackage{xcolor}

\usepackage{amssymb,amsthm,mathtools}
\usepackage{enumitem}
\usepackage{multirow}
\usepackage{array}
\usepackage{comment}
\usepackage{xcolor}

\newtheorem{theorem}{Theorem}
\newtheorem{proposition}{Proposition}
\newtheorem{lemma}{Lemma}
\newtheorem{corollary}{Corollary}
\theoremstyle{definition}

\theoremstyle{remark}

\theoremstyle{definition}
\newtheorem{discussion}{Discussion}
\renewcommand{\eqref}[1]{(\ref{#1})}

\newcommand{\Prob}{\mathbb{P}}
\newcommand{\N}{\mathcal{N}}
\newcommand{\Reg}{\mathrm{Reg}}
\newcommand{\ReMax}{\mathrm{ReMax}}

\newcommand{\ind}{\mathbf{1}}

\newcommand{\dd}{\mathrm{d}}
\newcommand{\supp}{\mathrm{supp}}

\newcommand{\simplex}{\Delta^{K-1}}

\newcounter{algorithmctr}

\allowdisplaybreaks

\title{Regret Analysis of Retry-Based Bandits}

\author{
Bingkui Tong$^{1}$ \qquad
Junpei Komiyama$^{1,2}$ \qquad
Soichiro Nishimori$^{3}$ \qquad
Paavo Parmas$^{3}$ \\[0.5em]
$^{1}$Mohamed bin Zayed University of Artificial Intelligence \\
$^{2}$RIKEN AIP \\
$^{3}$The University of Tokyo \\[0.7em]
\texttt{bingkui.tong@mbzuai.ac.ae} \qquad
\texttt{junpei@komiyama.info} \\
\texttt{nishimori@ms.k.u-tokyo.ac.jp} \qquad
\texttt{paavo.parmas@weblab.t.u-tokyo.ac.jp}
}

\begin{document}

\maketitle

\begin{abstract}
We provide the first regret analysis of ReMax in stochastic multi-armed bandits. Originally introduced for reinforcement learning, ReMax is motivated by the role of exploration when multiple attempts are allowed, and is closely related to retry-based objectives such as pass@$k$ and max@$k$, which value the best outcome across those attempts. Given a posterior over arm values, ReMax chooses a sampling distribution that maximizes the posterior expected maximum arm value over $M$ virtual draws. For Gaussian rewards with $K$ arms and any fixed integer $M\ge2$, we characterize optimal sampling distributions through an expected-improvement balance condition and prove an $O(\sqrt{(M-1)KT\log T})$ finite-time bound on expected regret. Our analysis separates a logarithmic contribution from rounds without optimal-arm underestimation and a recovery term that dominates the bound. For every fixed instance with pairwise distinct arm means, we further prove an asymptotic logarithmic regret upper bound whose coefficient is at most $M-1$ times the optimal coefficient by Lai and Robbins. When $M=2$, this bound implies the asymptotic optimality of ReMax. Gaussian-bandit experiments show that ReMax performs competitively with Thompson sampling and KL-UCB. They also reveal a trade-off: larger $M$ reduces regret incurred during optimal-arm underestimation, but this reduction does not necessarily translate into lower total regret. Controlled recovery experiments show that separating suboptimal means accelerates recovery, while severe underestimation of the optimal arm can lead to long delays before it is sampled again.
\end{abstract}

\section{Introduction}

The stochastic multi-armed bandit problem is a fundamental framework for studying the trade-off between exploration and exploitation in sequential decision making. In the classical formulation, the learner chooses one arm per round and aims to maximize cumulative reward. Accordingly, most bandit algorithms select actions according to quantities that estimate, upper-bound, or sample from the value of a single pull. Canonical examples include optimism-based methods such as upper confidence bound (UCB) \citep{lairobbins1985,Auer2002,cappe2013}, empirical divergence methods \citep{DBLP:conf/colt/HondaT10,maillard_phd,DBLP:conf/aistats/BianJ22}, and Thompson sampling \citep{thompson1933likelihood,chapelle2011empirical,DBLP:conf/alt/KaufmannKM12,komiyama15ts,agrawal2017near}.

Retry-based objectives, which value the best outcome across multiple attempts rather than the outcome of a single attempt, provide an alternative mechanism for inducing exploration. In reinforcement learning, ReMax~\citep{koyamada2022remax,nishimori2026emergence} instantiates this idea by optimizing the expected maximum return over multiple virtual trials under reward uncertainty. Related objectives have recently been studied in language-model reasoning, where pass@$k$ and related max@$k$ criteria directly optimize performance over multiple generations~\citep{tang2025optimizing,chen2025passktraining,walder2026pass,bagirov2025bestofn}. Improvements in pass@$1$ do not necessarily translate into corresponding gains in pass@$k$~\citep{chen2026does}, further motivating objectives that explicitly account for multiple attempts. Together, these results suggest that retry-based objectives provide a promising mechanism for exploration. By studying this mechanism in a tractable bandit setting, we hope to provide theoretical insight into max@$k$-style optimization, including its emerging use in language-model reasoning.

Despite these empirical findings, theoretical guarantees for retry-induced exploration remain limited. Prior work has provided empirical evidence that ReMax adapts exploration to uncertainty in bandit settings and theoretical evidence that its objective induces exploration in the deterministic-reward setting~\citep{nishimori2026emergence,nishimori2026retry}. Yet how this exploration behavior translates into regret guarantees for posterior-based ReMax remains theoretically unclear.

We fill this gap by providing the first regret analysis of ReMax in stochastic multi-armed bandits. For Gaussian rewards and any fixed integer $M \ge 2$, its optimal sampling distribution satisfies an expected-improvement balance condition: all arms assigned positive probability have the same posterior-averaged marginal contribution to the maximum over the other $M-1$ virtual draws, while no arm outside the support has a larger contribution. Using this characterization, we first prove an $O(\sqrt{(M-1)KT\log T})$ finite-time regret bound. The leading term comes from recovery after the optimal arm is underestimated and consequently sampled too rarely. This recovery can be particularly slow when suboptimal arms have similar empirical means, because the expected improvements between them decay slowly and can sustain their sampling probabilities. When all arm means are pairwise distinct, the expected improvement of a lower-mean suboptimal arm over a higher-mean one instead decays exponentially after sufficient sampling. Under this additional assumption, we prove an asymptotic $O(\log T)$ regret bound. Its leading coefficient is at most $M-1$ times the asymptotically optimal coefficient for general $M$; when $M=2$, it matches the asymptotically optimal coefficient \citep{lairobbins1985}, establishing that ReMax is asymptotically optimal in this setting.

\paragraph{Contributions}
\begin{enumerate}[leftmargin=*, itemsep=1pt, topsep=2pt, parsep=0pt]
    \item We characterize the optimal sampling distribution of the ReMax objective through a Bayesian expected-improvement balance condition, which clarifies the structure of the ReMax sampling rule.
    \item For Gaussian rewards and any fixed $M\ge2$, we prove an $O(\sqrt{(M-1)KT\log T})$ finite-time regret bound and identify recovery from optimal-arm underestimation as the leading term. On any fixed Gaussian instance with pairwise-distinct means, we further prove a logarithmic asymptotic upper bound with coefficient at most $M-1$ times the coefficient of Lai and Robbins (1985). For $M=2$, the coefficient is exact, and ReMax is asymptotically optimal on this instance class.
    \item We evaluate ReMaxGrad, a fixed-budget stochastic-gradient approximation to ReMax, on fixed and random Gaussian bandits, showing that larger $M$ reduces underestimation regret without monotonically improving total regret. Controlled experiments with analytical $M=2$ ReMax illustrate how suboptimal-arm separation and the severity of initial underestimation affect recovery.
\end{enumerate}

\section{The ReMax Bandit Algorithm}
\label{sec:algorithm}

\subsection{Bandit Setting}
\label{sec:prelim}
We consider a $K$-armed stochastic bandit. We reserve $\mu_i$ for the true mean reward of arm $i$, $\hat{\mu}_i(t)$ for its empirical mean at round $t$, and $\theta_i$ for a posterior sample of arm $i$. At each round $t \in \{1,\dots,T\}$, the learner chooses an arm $I(t)\in [K] := \{1,\dots,K\}$ and observes a reward $R_t$ with mean $\mu_{I(t)}$. We assume that the means are ordered as $\mu_1 > \mu_2 \ge \cdots \ge \mu_K$, so arm $1$ is uniquely optimal. For each $i>1$, let the suboptimality gap be $\Delta_i := \mu_1-\mu_i$ and $N_i(t) := \sum_{s=1}^{t-1} \ind\{I(s)=i\}$, and define the (pseudo) regret as $\Reg(T) := \sum_{t=1}^{T}(\mu_1-\mu_{I(t)}) = \sum_{i=2}^{K} \Delta_i N_i(T+1)$. Thus $N_i(t)$ is the number of observations of arm $i$ available before the decision at round $t$. Throughout the paper, we let $\Pi_t$ denote the posterior distribution over arm means after round $t$ and write $\bm\theta=(\theta_1,\dots,\theta_K)$ for a sample drawn from the posterior.

\subsection{Objective and Sampling Rule}

Let $\pi\in\simplex$ be a distribution over the $K$ arms, and let $A_{1:M}\sim\pi$ denote $M$ i.i.d. draws from $\pi$. The ReMax objective~\citep{nishimori2026emergence, koyamada2022remax} is
\begin{equation}
\label{eq:remax-objective}
J_{\ReMax}^{(M)}(\pi;\Pi)
:=
\mathbb{E}_{\bm\theta\sim\Pi}
\left[
\mathbb{E}_{A_{1:M}\sim\pi}
\left[
\max_{m\in[M]} \theta_{A_m}
\,\middle|\, \bm\theta
\right]
\right].
\end{equation}
When the same arm appears multiple times in $A_{1:M}$, we reuse the same posterior sample $\theta_i$ for that arm. Thus, repeated occurrences of an arm correspond to repeated action selection rather than re-sampling its posterior value.

At round $t$, the ReMax bandit algorithm selects a distribution maximizing the ReMax objective under the current posterior:
\begin{equation}
\label{eq:remax-policy}
\pi_t \in \argmax_{\pi\in\simplex} J_{\ReMax}^{(M)}(\pi;\Pi_{t-1}),
\qquad
I(t) \sim \pi_t.
\end{equation}
After observing the reward $R_t$, the posterior is updated to obtain $\Pi_t$.

\subsection{Expected-Improvement Balance}

A key structural property of ReMax is that, under any optimal sampling distribution, all arms in the support have the same posterior-averaged marginal improvement. For a fixed posterior sample $\bm\theta$, let
\begin{equation}
W_{M-1} := \max_{m\in[M-1]} \theta_{A_m},
\qquad A_{1:M-1}\sim\pi_t.
\end{equation}
We define $(x)_+ := \max\{x,0\}$ and the posterior-averaged marginal expected improvement (EI) of arm $i$ at the optimal sampling distribution $\pi_t$ by
\begin{equation}
\bar s_{t,i}
:=
\mathbb{E}_{\bm\theta\sim\Pi_{t-1}}
\left[
\mathbb{E}_{A_{1:M-1}\sim\pi_t}
\bigl[(\theta_i-W_{M-1})_+\bigr]
\right].
\end{equation}

\begin{proposition}[Concavity and Expected-Improvement Balance]
\label{prop:ei-balance}

For every integer $M\ge2$, the ReMax objective $\pi\mapsto J_{\ReMax}^{(M)}(\pi;\Pi_{t-1})$ is concave on $\simplex$. Consequently, the simplex Karush--Kuhn--Tucker (KKT) conditions are sufficient for global optimality. Fix a round $t$ and let $\pi_t$ satisfy Equation~\eqref{eq:remax-policy}. Then:
\begin{enumerate}[leftmargin=1.5em]
    \item if $i,j\in\supp(\pi_t)$, then $\bar s_{t,i}=\bar s_{t,j}$;
    \item if $j\in\supp(\pi_t)$ and $k\notin\supp(\pi_t)$, then $\bar s_{t,j}\ge \bar s_{t,k}$,
\end{enumerate}
where $\supp(\pi_t)=\{i\in[K]:\pi_{t,i}>0\}$.
\end{proposition}

The proof follows by differentiating the ReMax objective with respect to the sampling weights and applying the KKT conditions on the simplex. See Appendix~\ref{app:ei-proof} for the full proof.

For the general best-of-$M$ objective, draw $M-2$ additional baseline indices from $\pi_t$ and write
\[
    W_{M-2}:=\max_{m\in[M-2]}\theta_{A_m},
    \qquad W_0:=-\infty.
\]
Define the contextual improvement $H_{ij}(t):=\E_{\bm\theta,A_{1:M-2}}\left[\bigl(\theta_i-\max\{\theta_j,W_{M-2}\}\bigr)_+\right]$. Splitting one of the $M-1$ baseline actions gives
\begin{equation}
\label{eq:EI_weighted_sum}
\bar s_{t,i} = \sum_{j=1}^{K} \pi_{t,j} H_{ij}(t).
\end{equation}
Thus, arms in the support equalize weighted contextual improvements. When $M=2$, $H_{ij}(t)$ reduces to the pairwise quantity
\[
    G_{ij}(t):=\E[(\theta_i-\theta_j)_+],
    \qquad i\ne j,
\]
and we set $G_{ii}(t):=0$ under the duplicate-sample convention.

\begin{discussion}[Difference from Conventional Expected Improvement]
Note that the expected-improvement quantity in Proposition~\ref{prop:ei-balance} differs from the expected improvement algorithms studied in prior work, such as \cite{ryzhov2016}, which are widely used in global optimization \citep{Mockus1989,Vazquez10,brochu2010tutorial,JMLR:v12:bull11a,Snoek2012}. Proposition~\ref{prop:ei-balance} measures how adding the $M$-th arm improves the maximum over the other $M-1$ arms in the same round $t$, whereas classical expected improvement algorithms measure how a single sample improves upon the maximum of the previous $t-1$ samples (across rounds). In fact, the analysis in \cite{ryzhov2016} implies that expected improvement leads to sublinear regret, although its regret guarantees are not as strong as those of other popular bandit algorithms.
\end{discussion}

\subsection{Posterior-Variance Adaptation}
\label{subsec:reward_variance_adaptation}

For $M=2$ and distinct arms $i,j$ with Gaussian posteriors
$\Pi_{t-1,k}=\N(\hat\mu_k(t),1/N_k(t))$, let
$\delta_{ij}(t):=\hat\mu_i(t)-\hat\mu_j(t)$,
$\sigma_{ij}^2(t):=1/N_i(t)+1/N_j(t)$, and
$\gamma_{ij}(t):=\delta_{ij}(t)/\sigma_{ij}(t)$. Then, by
Lemma~\ref{lem:gaussian-positive-part},
\begin{equation}
\label{eq:gij-closed-form}
G_{ij}(t)
=
\sigma_{ij}(t)\phi(\gamma_{ij}(t))
+
\delta_{ij}(t)\Phi(\gamma_{ij}(t)),
\end{equation}
where $\phi$ and $\Phi$ are the standard normal density and CDF.

\begin{wrapfigure}{r}{0.45\linewidth}
\vspace{-1.5em}
\centering
\includegraphics[width=\linewidth]{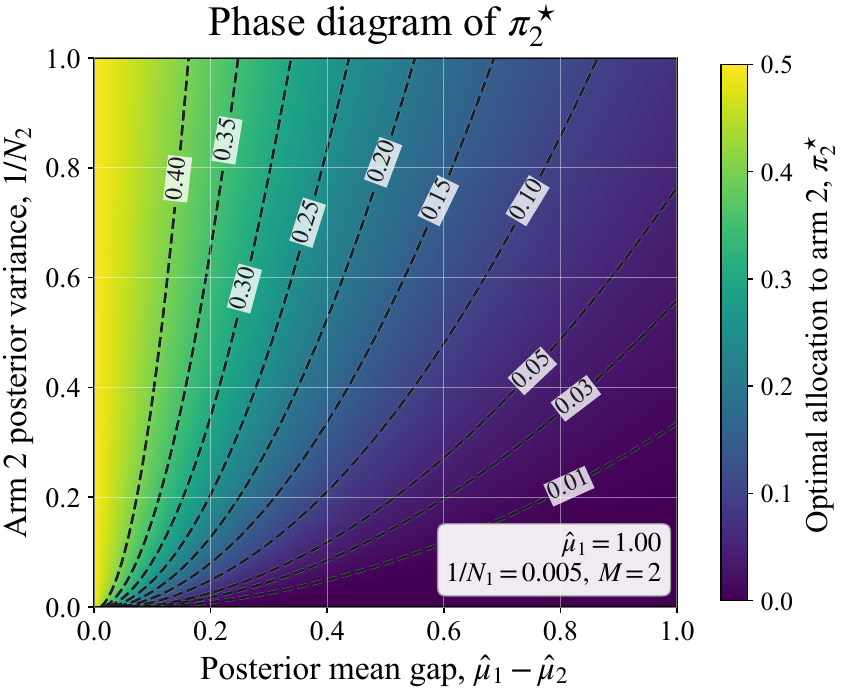}
\vspace{-2em}
\caption{Optimal sampling probability $\pi_2^\star$ versus posterior mean gap and arm 2's posterior variance ($K=M=2$).}
\label{fig:reward_variance_adaptation}
\vspace{-1em}
\end{wrapfigure}

Equation~\eqref{eq:gij-closed-form} shows how posterior uncertainty directly affects expected improvement: the first term scales with the posterior standard deviation, so an uncertain arm can remain valuable even when its posterior mean is suboptimal. Figure~\ref{fig:reward_variance_adaptation} illustrates this effect for $K=M=2$. Holding arm 1 fixed, the optimal sampling probability of arm 2 increases with its posterior variance and decreases with the posterior mean gap $\hat\mu_1-\hat\mu_2$. Thus, ReMax naturally allocates more probability to uncertain arms whose upper tails can still improve the best-of-two objective.

For general $M$, the same upper-tail effect enters the contextual quantity $H_{ij}(t)$ in Equation~\eqref{eq:EI_weighted_sum}. The relevant improvement is then measured relative to the maximum of the other virtual draws rather than a single competing arm. In contrast, Thompson sampling uses the posterior probability that an arm is optimal, whereas ReMax also accounts for the magnitude of its possible improvement over the virtual maximum.

\section{Regret Analysis}
\label{sec:theory}
We analyze the frequentist regret of ReMax for any fixed $M\ge2$.\footnote{Bayesian regret follows by integrating the frequentist bound over a prior on the arm means.}  Throughout this section, expectations are taken over the reward noise
and the randomness of the algorithm, with the bandit instance held fixed.

\subsection{Main Results}
\label{subsec:main_results}

\begin{theorem}[Finite-time regret bound for ReMax]
\label{thm:main}
Let $\Delta_{\max}:=\max_{i\ge2}\Delta_i$ and $\Delta_{\min}:=\min_{i\ge2}\Delta_i$. Fix an integer $M\ge2$, write $r:=M-1$, and assume $T\ge K$.
Suppose that rewards are mutually independent across arms and pulls,
with every pull of arm $i$ distributed as $\mathcal N(\mu_i,1)$.
Assume that each arm is initially pulled once and that, at every
subsequent round, Equation~\eqref{eq:remax-policy} is solved exactly
under the joint posterior
\[
    \Pi_{t-1}
    =
    \bigotimes_{i=1}^K
    \mathcal N\!\left(
        \hat\mu_i(t),\frac{1}{N_i(t)}
    \right).
\]
For any $0<\varepsilon\le\min\{1,\Delta_{\min}/5\}$ and $D_j^\varepsilon:=\frac{(\Delta_j-\tfrac32\varepsilon)^2}{2}$, there is a universal constant $C>0$ such that
\begin{align}
    \E[\Reg(T)]
    \le{}&
    \sum_{j=2}^K\Delta_j
    +
    \frac{C}{\varepsilon^2} \sum_{j=2}^K\Delta_j
    +
    4r\Delta_{\max}
    \sum_{i=2}^K
    \frac{1}{\exp((\Delta_i-\varepsilon)^2/2)-1}
    \notag\\
    &+
    \underbrace{
        C r \sum_{j=2}^K \frac{\Delta_j}{D_j^\varepsilon}
        \log\!\left(
            1+2T
            \left(
                \frac{4}{\sqrt{\pi}\,\Delta_j}
            \right)^{\!1/r}
        \right)
    }_{\textit{Saturation}}
    +
    \underbrace{
        C\Delta_{\max}\sqrt{rKT\log(T+1)}\,
        \frac{\log(2/\varepsilon)}{\varepsilon^3}
    }_{\textit{Underestimation}}.
    \label{eq:generalM-main-bound}
\end{align}
Consequently, for every fixed $M$ and fixed bandit instance,
$\E[\Reg(T)]=O(\sqrt{T\log T})=o(T)$.
\end{theorem}

Theorem~\ref{thm:main} has a logarithmic saturation term and a square-root underestimation term. Saturation is more delicate than in Thompson sampling because the EI comparison becomes effective only after both the suboptimal arm and arm \(1\) are sufficiently sampled. The leading term arises when an unfavorable empirical mean makes the optimal arm slow to re-enter the ReMax support. For intuition, consider the \(M=2\) three-arm instance \(\mu_1>\mu_2=\mu_3\) with \(\hat\mu_2=\hat\mu_3=\mu_2>\hat\mu_1\). Mutual EI between arms \(2\) and \(3\) decays as \(t^{-1/2}\) when both counts are of order \(t\). Balancing this quantity against the Gaussian upper-tail EI of arm \(1\) can delay correction of its empirical mean. The proof controls the same mechanism through the virtual-maximum EI bound and a log-count argument. Section~\ref{sec:three-arm-recovery} provides a more direct illustration of the mechanism underlying this term. Note that the regret bound of Theorem \ref{thm:main} is distribution-dependent as it depends on $\varepsilon = \Theta(\Delta_{\min})$, and thus we cannot directly convert it to a distribution-independent bound \citep{DBLP:journals/jmlr/AudibertB10}.

The finite-time theorem allows equal suboptimal means. The preceding example indicates why a sharper conclusion becomes possible when the suboptimal means are separated: after sufficiently many samples, \(\min\{G_{ij}(t),G_{ji}(t)\}\) decreases exponentially rather than at order \(t^{-1/2}\). We now state the resulting pointwise asymptotic bound. 

\begin{theorem}[Asymptotic regret bound for ReMax]
\label{thm:asymptotic-regret}
Fix an integer $M\ge2$ and write $r:=M-1$.
Under the assumptions of Theorem~\ref{thm:main}, suppose in addition
that the arm means are pairwise distinct,
$\mu_1>\mu_2>\cdots>\mu_K$.
Then, for every $j>1$,
\begin{equation}
    \limsup_{T\to\infty}
    \frac{\E[N_j(T+1)]}{\log T}
    \le
    \frac{2r}{\Delta_j^2}.
    \label{eq:asymptotic-count-upper}
\end{equation}
Consequently,
\begin{equation}
    \limsup_{T\to\infty}
    \frac{\E[\Reg(T)]}{\log T}
    \le
    2r\sum_{j=2}^K\frac{1}{\Delta_j}.
    \label{eq:asymptotic-regret-upper}
\end{equation}
\end{theorem}
It is known that, for any algorithm that is uniformly good\footnote{An algorithm is uniformly good if its regret is $o(T^a)$ for all $a>0$ and all instances.}, the following asymptotic lower bounds hold \citep{lairobbins1985}:
\begin{align}
    \liminf_{T\to\infty}
    \frac{\E[N_j(T+1)]}{\log T}
    &\ge
    \frac{2}{\Delta_j^2},
    \qquad j>1.
\end{align}
Together with the upper bounds in Theorem~\ref{thm:asymptotic-regret}, this shows that when $M=2$, ReMax matches the Lai--Robbins lower bound and is therefore asymptotically optimal on every instance with pairwise-distinct arm means.

\subsection{Proof Sketch}

This section provides a proof sketch for the key idea for the two main theorems. One may skip this section.
Write $r=M-1$. For the finite-time bound, we separate rounds according to whether $\hat\mu_1(t)\le\mu_1-\varepsilon$, where $3 \varepsilon$ is smaller than the smallest suboptimality gap $\min_{j>1}\Delta_j$. Similarly to the analysis of Thompson sampling \citep{DBLP:conf/alt/KaufmannKM12,agrawalgoyalaistats2013}, the main difficulty is that an underestimated optimal arm may receive very little sampling probability, allowing the same unfavorable estimate to persist. The proof for ReMax is much more involved because multiple posterior samples interact through the virtual maximum.
The full proofs appear in Appendices~\ref{app:full_proof} and~\ref{app:asymptotic-proof}.

\paragraph{Saturation Term}
Consider rounds on which the optimal arm is not substantially underestimated, namely $\hat\mu_1(t)>\mu_1-\varepsilon$. For a suboptimal arm $j$, concentration bounds the expected number of pulls with $\hat\mu_j(t)>\mu_j+\varepsilon$ by $2/\varepsilon^2$. On the remaining rounds, $\hat\mu_1(t)-\hat\mu_j(t)>\Delta_j-2\varepsilon>0$, so the empirical means correctly rank arm $1$ above arm $j$.

The key point is that this empirical gap makes arm $j$ unlikely to be selected once both arms have enough observations. For each $j$, choose a sample threshold $L_j=O(\log T)$ large enough that the EI balance condition gives $\pi_{t,j}\le 1/(2T)$ whenever $N_1(t),N_j(t)\ge L_j$ on these rounds. We then count pulls of arm $j$ in three cases. Before $N_j(t)$ reaches $L_j$, there are at most $L_j$ such pulls. If $N_j(t)\ge L_j$ but $N_1(t)<L_j$, the EI comparison gives $\pi_{t,j}\le\pi_{t,1}$. Thus, the expected number of pulls of arm $j$ in this case is at most the expected number of pulls of arm $1$ in the same case, which is at most $L_j$. Finally, when both counts are at least $L_j$, summing $\pi_{t,j}\le 1/(2T)$ over at most $T$ rounds gives at most $1/2$ expected pulls. Combining these cases with the overestimation bound, and weighting each arm's pull count by its gap $\Delta_j$, yields the logarithmic term in Theorem~\ref{thm:main}.

\paragraph{Underestimation Term}

As in the analysis of Thompson sampling, underestimation of the optimal arm requires more detailed analysis of the sampling dynamics than overestimation.
Group the underestimated rounds into blocks with $N_1(t)=n$. Within such a block, the empirical mean $\hat\mu_{1,n}$ of the first $n$ rewards from arm $1$ and its posterior remain unchanged. Let $e_n$ denote its posterior expected improvement above $\mu_1-\varepsilon$.

We first separate rounds on which sufficiently sampled but empirically optimistic suboptimal arms carry substantial sampling probability. Their expected number is controlled by charging them to upward deviations of suboptimal empirical means. On the remaining rounds, either substantial probability lies on insufficiently sampled arms, or EI balance forces the sampled arms to retain enough posterior uncertainty. In both cases, the log-count potential $\Psi_t:=\sum_{i=1}^K\log N_i(t)$ has conditional expected increment at least $c e_n^2/r$, for a universal constant $c>0$. Since its total increase is at most $K\log(T+1)$, the conditional expected number of remaining rounds in the count-$n$ block is at most $B_T/e_n^2$, where $B_T=C rK\log(T+1)$ for a universal constant $C>0$. Combining this with the horizon bound gives
\[
    \min\!\left\{T,\frac{B_T}{e_n^2}\right\}
    \le
    \frac{\sqrt{B_TT}}{e_n}.
\]
This step produces the square-root dependence on $T$. Averaging over $\hat\mu_{1,n}\sim\mathcal N(\mu_1,1/n)$ and summing over $n$ bounds the expected number of these rounds by
\[
    C\sqrt{rKT\log(T+1)}
    \frac{\log(2/\varepsilon)}{\varepsilon^3}.
\]
Multiplying the round-count bounds by $\Delta_{\max}$ gives the underestimation contribution in Theorem~\ref{thm:main}.

\paragraph{Asymptotic Regret}
For Theorem~\ref{thm:asymptotic-regret}, pairwise separation allows a sharper recovery argument. Fix $b\in(\max_{i>1}\mu_i,\mu_1)$. After sufficiently many observations, the expected improvement of a lower-mean suboptimal arm over a higher-mean one decays exponentially in their smaller sample count. Choose sampling thresholds so that these improvements are small relative to the optimal arm's EI above $b$. EI balance then forces some sampling probability onto arm $1$ or a suboptimal arm that has not yet reached its threshold. Each such draw either ends the current count-$n$ block or moves a suboptimal arm closer to its threshold.

The resulting waiting-time bound involves $1/\mathrm{EI}$, up to polynomial factors, where $\mathrm{EI}$ is the optimal arm's expected improvement above $b$. To see why this bound is integrable, write $\varepsilon_b:=\mu_1-b>0$ and $\hat\mu_{1,n}=b-y$ with $y\ge0$. The Gaussian EI bound gives
\[
    \frac{1}{\mathrm{EI}}
    \le C\sqrt n(1+ny^2)e^{ny^2/2}.
\]
Meanwhile, the density of $\hat\mu_{1,n}\sim\N(\mu_1,1/n)$ at $b-y$ is proportional to $\sqrt n\,e^{-n(\varepsilon_b+y)^2/2}$. Thus, when averaging the waiting-time bound, the exponential factors combine as
\[
    e^{ny^2/2}e^{-n(\varepsilon_b+y)^2/2}
    =e^{-n\varepsilon_b^2/2}e^{-n\varepsilon_b y}.
\]
The remaining decay in $y$ makes the polynomial factors integrable. Including these factors and integrating over $y\ge0$ bounds the expected contribution from the block with $N_1(t)=n$ by $C_b n^2 e^{-n\varepsilon_b^2/2}$, where $C_b$ is independent of $n$ and the horizon. Summing over $n$ therefore gives
\begin{equation}\label{ineq_finite_underestimation}
    \E\!\left[
        \sum_{t=K+1}^{\infty}
        \ind\{\hat\mu_1(t)\le b\}
    \right]
    \le C_b\sum_{n=1}^{\infty}n^2e^{-n\varepsilon_b^2/2}
    <\infty.
\end{equation}
Thus the expected underestimation contribution is bounded independently of the horizon.

To obtain the leading logarithmic coefficient of the saturation term, we need a sharper comparison than that of Theorem \ref{thm:main} because we must optimize the constant factor. The analysis here is more challenging than Thompson sampling because we must lower-bound $N_1(t)$ unlike in Thompson sampling, where only the suboptimal arm's count needs to be controlled.
Intuitively, the optimal arm should receive far more observations than a suboptimal arm, so that $N_1(t)\gg N_j(t)$: Its posterior uncertainty is then negligible relative to that of arm $j$, and the effective sample count in the pairwise EI bound satisfies
\[
    \frac{N_1(t)N_j(t)}{N_1(t)+N_j(t)}
    =\frac{N_j(t)}{1+N_j(t)/N_1(t)}
    \approx N_j(t).
\]
Thus the leading exploration cost is determined by the observations needed to distinguish arm $j$ from an essentially known optimal mean. 
However, we need to carefully control the rounds on which this is not yet the case. The proof makes this intuition precise by controlling separately the rounds on which the count ratio is not large. Fix $j>1$, sufficiently small $\varepsilon>0$, and $\eta>0$, and set
\[
    d_{j,\varepsilon}
    :=\frac{(\Delta_j-2\varepsilon)^2}{2},
    \qquad
    m_T
    :=\left\lceil
        \frac{r(1+\eta)\log T}{d_{j,\varepsilon}}
    \right\rceil,
\]
where $\eta>0$ is a small constant.
On rounds with $\hat\mu_1(t)>\mu_1-\varepsilon$ and $\hat\mu_j(t)\le\mu_j+\varepsilon$, the saturation inequality gives
\[
    \pi_{t,j}
    \le
    A\exp\!\left\{
        -\frac{d_{j,\varepsilon}}{r}
        \frac{N_1(t)N_j(t)}{N_1(t)+N_j(t)}
    \right\},
\]
where $A>0$ is independent of $t$ and $T$. There are at most $m_T$ pulls before arm $j$ reaches this threshold. Optimal-arm underestimation and empirical overestimation of arm $j$ contribute only $O(1)$ expected pulls.

For the remaining rounds, choose a sufficiently large fixed $\kappa$ such that $(1+\eta)\kappa/(\kappa+1)>1$, and split according to whether $N_1(t)\ge\kappa N_j(t)$. In this case, the effective sample count is at least $\kappa N_j(t)/(\kappa+1)$, so after $N_j(t)\ge m_T$ the sampling probability is at most $A T^{-(1+\eta)\kappa/(\kappa+1)}$. Summing over at most $T$ rounds gives an $o(1)$ contribution. In the complementary case, group rounds by $N_1(t)=n$. The proof bounds the lengths of non-underestimated blocks by a common integrable random bound; combining this with the sampling-probability bound gives an $O(1)$ contribution after saturation. Thus, rounds on which $N_1(t)$ is not much larger than $N_j(t)$ do not change the leading logarithmic coefficient. Hence
\[
    \E[N_j(T+1)]\le m_T+O(1).
\]
Dividing by $\log T$, then letting $\eta\downarrow0$ and $\varepsilon\downarrow0$, gives the coefficient $2r/\Delta_j^2$. Summing over suboptimal arms yields the regret bound. When $M=2$, the upper bound matches the Lai--Robbins lower bound, giving the stated limits and asymptotic optimality.

\begin{figure*}[tp]
    \centering
    \includegraphics[width=\linewidth]{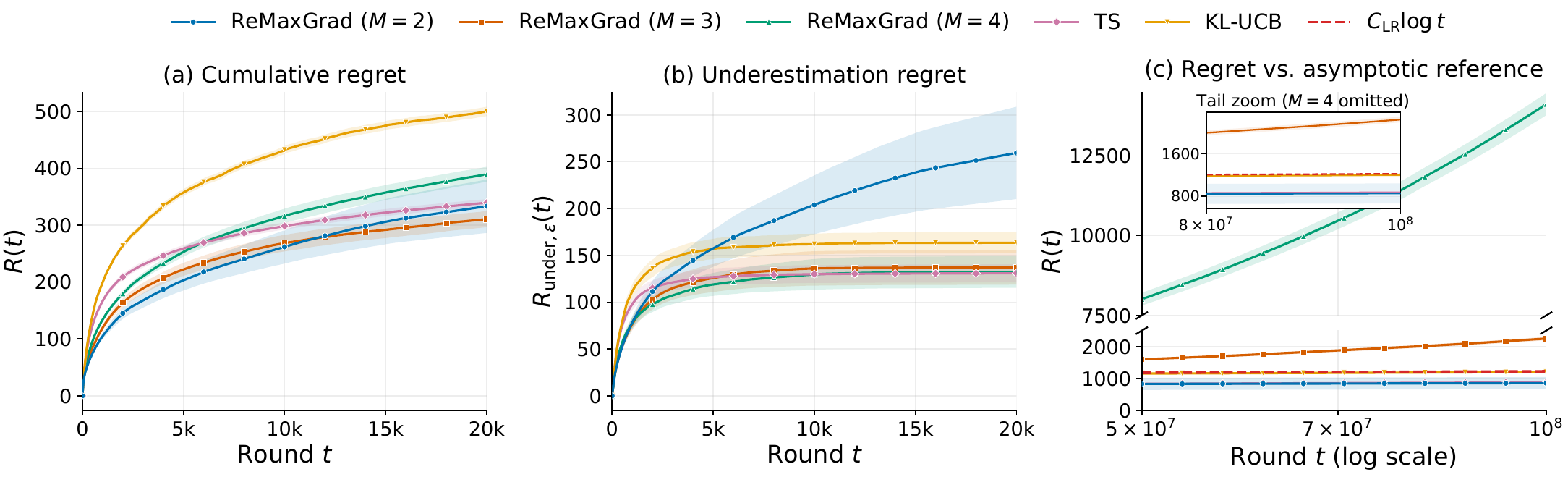}
    \vspace{-2em}
    \caption{Regret means over $100$ trajectories with one-standard-error bands. The red dashed line is $C_{\mathrm{LR}}\log t$, with $C_{\mathrm{LR}}\approx66.36$. (a) Cumulative and (b) underestimation regret through $t=20{,}000$. (c) Regret for $5\times10^7\le t\le10^8$ on a log time axis; the inset shows $8\times10^7\le t\le10^8$ without $M=4$. In (c), TS largely overlaps ReMaxGrad ($M=2$).}
    \label{fig:fixed-frequentist}
\end{figure*}

\section{Experiments}
\label{sec:experiment}

We evaluate ReMax on fixed and randomly generated Gaussian bandits and examine recovery from optimal-arm underestimation in two controlled three-arm experiments. The recovery experiments vary the separation between suboptimal means and the severity of the initial underestimation. We use ReMaxGrad, a stochastic-gradient approximation to the ReMax optimizer, with retry budgets $M\in\{2,3,4\}$ for the regret comparisons, and the analytical $M=2$ solution for the recovery experiments. Appendix~\ref{app:remaxgrad-implementation} provides implementation details and parameter settings.

\subsection{Experimental protocol}
\label{sec:experimental-protocol}
The fixed- and random-instance experiments use independent rewards from $\N(\mu_i,1)$, and each method begins by pulling every arm once. We compare ReMaxGrad~\citep{nishimori2026emergence} with Gaussian Thompson sampling (TS) and Gaussian KL--UCB. ReMaxGrad and TS use the flat-prior Gaussian posterior $\Pi_{t-1,i}=\N(\hat\mu_i(t),1/N_i(t))$, whereas Gaussian KL--UCB selects an arm maximizing the index $\hat\mu_i(t)+\sqrt{2\log(t-1)/N_i(t)}$.

For each instance, let $i^\star\in\argmax_i\mu_i$ and $\mu_\star=\mu_{i^\star}$. We report cumulative pseudo-regret and the regret incurred during optimal-arm underestimation after initialization:
\[
\Reg(t):=\sum_{s=1}^{t}\bigl(\mu_\star-\mu_{I(s)}\bigr),
\qquad
\Reg_{\mathrm{under},\varepsilon}(t)
:=\sum_{s=K+1}^{t}\bigl(\mu_\star-\mu_{I(s)}\bigr)
\ind\!\left\{\hat\mu_{i^\star}(s)<\mu_\star-\varepsilon\right\}.
\]
The second metric isolates the contribution of optimal-arm underestimation to cumulative regret.

\subsection{Fixed frequentist instance}
\label{sec:fixed-frequentist}
We construct a fixed $K=16$ Gaussian instance with evenly spaced means, using $g=0.1$: $\mu_i=1.5-g(i-1)$, $i\in[16]$. The means are pairwise distinct, as required by Theorem~\ref{thm:asymptotic-regret}. Underestimation regret uses the threshold $\varepsilon=0.02$. We average over $100$ trajectories, each initialized with an independent seed, and continue the same trajectories to $T=10^8$.

\begin{figure*}[tp]
    \centering
    \includegraphics[width=\linewidth]{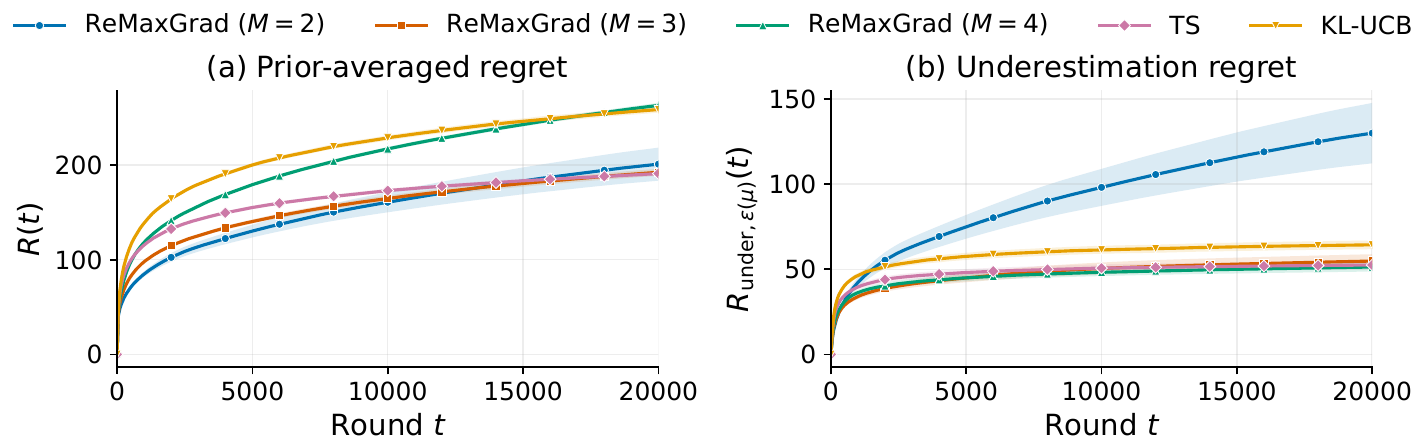}
    \vspace{-2em}
    \caption{Means over $1{,}000$ random $K=16$ instances ($\mu_i\overset{\mathrm{iid}}{\sim}\N(0,1)$), with one-standard-error bands. (a) Prior-averaged cumulative regret; (b) underestimation regret through $t=20{,}000$.}
    \label{fig:bayesian-experiment}
\end{figure*}

\paragraph{Results} Figure~\ref{fig:fixed-frequentist} (a) shows that $M=3$ attains the lowest mean cumulative regret at $t=20{,}000$, while $M=2$ remains close to TS. Figure~\ref{fig:fixed-frequentist} (b) shows that increasing $M$ reduces underestimation regret, although this reduction does not translate into a monotonic improvement in total regret. At large horizons, Figure~\ref{fig:fixed-frequentist} (c) shows that $M=2$ has a slope comparable to those of TS and KL--UCB on the logarithmic time axis, while the slopes become progressively steeper for $M=3$ and $M=4$. The observed behavior is consistent with the asymptotic result in Theorem~\ref{thm:asymptotic-regret}.

\subsection{Random-instance experiment}
\label{sec:bayesian-experiment}

We independently sample $1{,}000$ mean vectors for random $K=16$ instances according to $\mu_i\overset{\mathrm{iid}}{\sim}\N(0,1)$, where $i\in[16]$. Each mean vector defines one instance, on which every method is run once. Underestimation regret uses the instance-specific threshold $\varepsilon=\min\{1,\Delta_{\min}/5\}$, where $\Delta_{\min}:=\min_{i\ne i^\star}(\mu_\star-\mu_i)$. Averaging over instances and stochastic trajectories yields the prior-averaged regret.

\paragraph{Results} Figure~\ref{fig:bayesian-experiment} shows that $M=2$ and $M=3$ have mean cumulative regret comparable to TS and lower than KL--UCB at $t=20{,}000$. The $M=4$ variant performs similarly to KL--UCB. As in the fixed instance, increasing $M$ reduces underestimation regret without monotonically improving total regret.

\subsection{Three-arm recovery from underestimation}
\label{sec:three-arm-recovery}
We first examine how suboptimal-arm separation affects recovery. Consider the three-arm family $\bm\mu(\delta)=(1.5,1,1-\delta)$ for $\delta\in\{0,0.05,0.1,0.2\}$, with unit reward variance. Arm~$1$ has a constant gap of $0.5$ to the best suboptimal arm, while arms~$2$ and~$3$ are tied at $\delta=0$ and separated for $\delta>0$.

\begin{wrapfigure}{r}{0.48\linewidth}
\vspace{-1.5em}
\centering
\includegraphics[width=\linewidth]{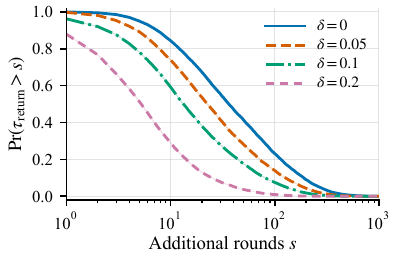}
\vspace{-2em}
\caption{Curves are estimated from $10{,}000$ trajectories per $\delta$, with shaded bands denoting one standard error. Each curve shows the probability that arm~$1$ remains unselected after $s$ additional rounds.}
\label{fig:three-arm-recovery}
\vspace{-2em}
\end{wrapfigure}

We prescribe an initial state with $16$ observations per arm and empirical means $(\hat\mu_1,\hat\mu_2,\hat\mu_3)=(0.8,1,1-\delta)$. Subsequent rewards and posterior updates follow the usual bandit dynamics, with flat-prior Gaussian posteriors. For each $\delta$, we simulate $10{,}000$ independently seeded trajectories and record $\tau_{\mathrm{return}}:=\inf\{s\geq1:I(s)=1\}$, the first return to arm~$1$ after initialization.

\paragraph{Results} Figure~\ref{fig:three-arm-recovery} shows the slowest recovery when the suboptimal arms are tied. Increasing their separation leads to progressively faster recovery. This illustrates how competition between suboptimal arms can prolong optimal-arm underestimation, and helps explain the role of pairwise separation in Theorem~\ref{thm:asymptotic-regret}.

\subsection{Empirical evidence for finite waiting times}
\label{sec:tail-waiting-time}

To empirically verify the finiteness of the underestimation term (i.e., \eqref{ineq_finite_underestimation}), we vary the severity of the initial underestimation. Here the three arms have means $(1.0,0.9,0.8)$ and reward standard deviation $\sigma=0.1$. Each arm is initially pulled once, and arm~$1$'s first observation is replaced by $x(q):=1+\sigma\Phi^{-1}(q)$, where $\Phi$ is the standard normal CDF. Thus, $q$ is the worst $q$-quantile of the optimal arm's reward distribution. 
For example, $q=0.01$ gives $x(q)\approx0.767$, an observation in the bottom 1\% of the reward distribution.
All other observations are sampled normally.
We use the flat-prior posterior $\N(\hat\mu_i,\sigma^2/N_i)$. 

We aim to measure the worst-case additional waiting time for the optimal arm to be revisited. 
To see this, let $\tau_{\rm return}:=\inf\{s\ge1:I(s)=1\}$ be the number of rounds until the algorithm draws arm 1 again. Let $m(q):=\E[\tau_{\mathrm{return}}\mid X_{1,1}=x(q)]$. 
To have a finite expected waiting time, the second draw of the optimal arm\footnote{One can also show that this quantity diverges for a UCB algorithm with insufficient exploration, where the exploration term is smaller than $\sqrt{2\sigma^2\log t/N_i(t)}$. Therefore, the finiteness of this quantity supports the sufficiency of exploration.} is required to occur quickly enough on average, which is captured by:
\begin{equation}\label{ineq_finite_seconddraw}
\int_0^1
\underbrace{m(q)}_{\text{expected waiting time given at the worst $q$-quantile}}
\, dq < \infty,
\end{equation}
and for this equation to hold, $m(q)$ must grow slower than $1/q$ as $q\to0$, as $\int (1/q) dq = \infty$.

\begin{wrapfigure}{r}{0.48\linewidth}
\vspace{-1.5em}
\centering
\includegraphics[width=\linewidth]{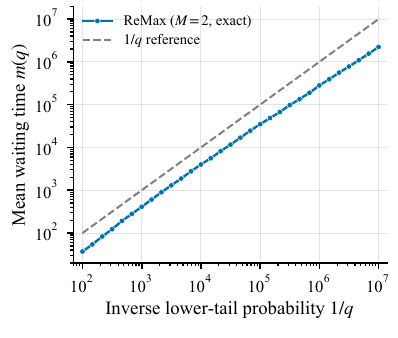}
\vspace{-3em}
\caption{Mean $m(q)$ over $1{,}000$ trajectories per $q$, with one-standard-error bands, plotted against $1/q$; the dashed line shows the $1/q$ reference.}
\label{fig:tail-waiting-time}
\end{wrapfigure}

To measure $m(q)$ empirically, we use $31$ logarithmically spaced values of $q$ from $10^{-2}$ to $10^{-7}$ and $1{,}000$ independently seeded trajectories per value, reusing the same seeds across values of $q$. Every trajectory returns before the cap of $10^8$ additional rounds.

\paragraph{Results} Figure~\ref{fig:tail-waiting-time} shows $m(q)$ as a function of $1/q$. We observe that $m(q)$ grows slower than $1/q$, empirically supporting Eq.~\eqref{ineq_finite_seconddraw}. At the same time, the mean recovery time increases from tens to millions of rounds as the initial observation moves deeper into the lower tail, approximately following the $1/q$ scale over the tested range. Thus, even with pairwise-distinct means, severe initial underestimation can substantially delay the return to the optimal arm.

\clearpage
\subsection*{AI use statement}

We used generative AI tools to suggest the approach underlying the asymptotic-optimality result in Theorem~\ref{thm:asymptotic-regret} and to assist in developing its proof. We also used generative AI tools to check and improve mathematical arguments, assist with experimental code, and edit and polish the manuscript. The AI-assisted proofs were subsequently reviewed line by line by multiple authors, who independently checked their correctness and revised them to improve clarity and readability. All AI-assisted code was also reviewed and verified by the authors.

\subsection*{Ethics statement}

We are not aware of any ethical concerns associated with this work.

\subsubsection*{Author Contributions}

\noindent\textbf{Paavo Parmas:}
Proposed the algorithm, created the initial implementation, and performed preliminary experiments on Gaussian bandit tasks. Derived the initial Expected Improvement balance condition and initiated the theoretical collaboration with Junpei Komiyama. Oversaw the project, contributed to discussions on the algorithm's efficacy, proposed variance inflation to mitigate the underestimation issue, and reviewed the manuscript.

\medskip
\noindent\textbf{Soichiro Nishimori:}
Led the final implementation and experimental evaluation. Contributed to discussions regarding the algorithm's mechanics and reviewed the manuscript.

\medskip
\noindent\textbf{Junpei Komiyama:}
Conducted the initial theoretical analysis and supervised the subsequent theoretical work.

\medskip
\noindent\textbf{Bingkui Tong:}
Completed the theoretical analysis and proofs based on the initial framework. Authored the manuscript.

\subsubsection*{Acknowledgments}
J. Komiyama was supported by the MBZUAI Start-up Fund [BF0121].

\clearpage
\bibliography{iclr2027_conference}
\bibliographystyle{iclr2027_conference}

\appendix
\clearpage
\section{Proof of Proposition~\ref{prop:ei-balance}}
\label{app:ei-proof}

\begin{proof}[Proof of Proposition~\ref{prop:ei-balance}]
Fix a round $t$ and abbreviate
\[
J(\pi)
:=
J_{\mathrm{ReMax}}^{(M)}(\pi;\Pi_{t-1}).
\]
For a fixed posterior sample $\theta$, define
\[
v(\pi\mid\theta)
:=
\mathbb{E}_{A_{1:M}\sim\pi}
\left[
\max_{m\in[M]}\theta_{A_m}
\right].
\]

We first prove concavity. Let
\[
x_1<\cdots<x_L
\]
be the distinct values among $\theta_1,\ldots,\theta_K$, and define
\[
Q_\ell(\pi)
:=
\sum_{i:\,\theta_i\le x_\ell}\pi_i,
\qquad
\ell\in[L].
\]
The maximum of the $M$ draws is at most $x_\ell$ with probability
$Q_\ell(\pi)^M$. Therefore,
\[
\begin{aligned}
v(\pi\mid\theta)
&=
x_1
+
\sum_{\ell=1}^{L-1}
(x_{\ell+1}-x_\ell)
\bigl(1-Q_\ell(\pi)^M\bigr) \\
&=
x_L
-
\sum_{\ell=1}^{L-1}
(x_{\ell+1}-x_\ell)Q_\ell(\pi)^M.
\end{aligned}
\]
Each $Q_\ell$ is affine on the simplex, and $q\mapsto q^M$ is
convex on $[0,1]$. Hence $v(\cdot\mid\theta)$ is concave on
$\Delta_{K-1}$. Taking expectations over $\theta\sim\Pi_{t-1}$
shows that $J$ is concave.

Expanding over all realizations of $a_1,\ldots,a_M\in[K]$ gives
\[
v(\pi\mid\theta)
=
\sum_{a_1=1}^K\cdots\sum_{a_M=1}^K
\left(
\prod_{m=1}^M\pi_{a_m}
\right)
\max\{\theta_{a_1},\ldots,\theta_{a_M}\}.
\]
The Gaussian posterior satisfies
\[
\mathbb{E}_{\theta\sim\Pi_{t-1}}
\left[
\max_{i\in[K]}|\theta_i|
\right]
<\infty.
\]
Thus, after averaging the finite expansion above, $J$ is a
polynomial in $\pi$ with finite coefficients. In particular, its
partial derivatives exist at both interior and boundary points of
the simplex, and differentiation may be passed through the
posterior expectation.

By the product rule,
\[
\frac{\partial}{\partial\pi_i}
\left(
\prod_{m=1}^M\pi_{a_m}
\right)
=
\sum_{k=1}^M
\mathbf{1}\{a_k=i\}
\prod_{m\ne k}\pi_{a_m}.
\]
Substituting this identity into the expansion gives
\begin{align}
\frac{\partial v(\pi\mid\theta)}{\partial\pi_i}
&=
\sum_{a_1,\ldots,a_M}
\left(
\sum_{k=1}^M
\mathbf{1}\{a_k=i\}
\prod_{m\ne k}\pi_{a_m}
\right)
\max\{\theta_{a_1},\ldots,\theta_{a_M}\}
\\
&=
\sum_{k=1}^M
\sum_{a_1,\ldots,a_M}
\mathbf{1}\{a_k=i\}
\left(
\prod_{m\ne k}\pi_{a_m}
\right)
\max\{\theta_{a_1},\ldots,\theta_{a_M}\}.
\end{align}
By exchangeability, the $M$ terms in the outer sum are identical.
For any fixed position $k$, the condition $a_k=i$ fixes that draw
to arm $i$, while the remaining $M-1$ indices are independently
drawn from $\pi$. Writing
\[
W_{M-1}
:=
\max_{m\in[M-1]}\theta_{A_m},
\qquad
A_{1:M-1}\sim\pi,
\]
we obtain
\[
\frac{\partial v(\pi\mid\theta)}{\partial\pi_i}
=
M\,
\mathbb{E}_{A_{1:M-1}\sim\pi}
\left[
\max\{\theta_i,W_{M-1}\}
\right].
\]

Define
\[
s_{t,i}(\pi,\theta)
:=
\mathbb{E}_{A_{1:M-1}\sim\pi}
\left[
(\theta_i-W_{M-1})_+
\right].
\]
Using
\[
\max\{\theta_i,W_{M-1}\}
=
(\theta_i-W_{M-1})_+ + W_{M-1},
\]
we have
\begin{align}
\frac{\partial v(\pi\mid\theta)}{\partial\pi_i}
&=
M\left(
\mathbb{E}_{A_{1:M-1}\sim\pi}
\left[
(\theta_i-W_{M-1})_+
\right]
+
\mathbb{E}_{A_{1:M-1}\sim\pi}
[W_{M-1}]
\right)
\\
&=
M\left(
s_{t,i}(\pi,\theta)
+
\mathbb{E}_{A_{1:M-1}\sim\pi}
[W_{M-1}]
\right).
\end{align}

Let
\[
C_t(\pi)
:=
\mathbb{E}_{\theta\sim\Pi_{t-1}}
\left[
\mathbb{E}_{A_{1:M-1}\sim\pi}
[W_{M-1}]
\right].
\]
Averaging over $\theta\sim\Pi_{t-1}$ gives
\begin{align}
\frac{\partial J(\pi)}{\partial\pi_i}
&=
\mathbb{E}_{\theta\sim\Pi_{t-1}}
\left[
\frac{\partial v(\pi\mid\theta)}{\partial\pi_i}
\right]
\\
&=
M\left(
\mathbb{E}_{\theta\sim\Pi_{t-1}}
[s_{t,i}(\pi,\theta)]
+
C_t(\pi)
\right).
\end{align}
At the optimizer $\pi_t$,
\[
\frac{\partial J(\pi_t)}{\partial\pi_i}
=
M\left(
\bar{s}_{t,i}+C_t(\pi_t)
\right),
\qquad
\bar{s}_{t,i}
:=
\mathbb{E}_{\theta\sim\Pi_{t-1}}
[s_{t,i}(\pi_t,\theta)].
\]

Since $J$ is differentiable and concave, the simplex KKT
conditions are sufficient for global optimality. At $\pi_t$, there
exists $\lambda_t\in\mathbb{R}$ such that
\[
\frac{\partial J(\pi_t)}{\partial\pi_i}
=
\lambda_t
\quad
\text{for }i\in\operatorname{supp}(\pi_t),
\]
and
\[
\frac{\partial J(\pi_t)}{\partial\pi_k}
\le
\lambda_t
\quad
\text{for }k\notin\operatorname{supp}(\pi_t).
\]
Since $C_t(\pi_t)$ is common to all arm indices, these relations
imply
\[
\bar{s}_{t,i}
=
\bar{s}_{t,j}
\quad
\text{for }i,j\in\operatorname{supp}(\pi_t),
\]
and
\[
\bar{s}_{t,j}
\ge
\bar{s}_{t,k}
\quad
\text{for }j\in\operatorname{supp}(\pi_t),
\quad
k\notin\operatorname{supp}(\pi_t).
\]
This proves the claims.
\end{proof}

\section{Proof of Theorem~\ref{thm:main}}
\label{app:full_proof}

\subsection{Setup and decomposition}

Throughout this appendix, fix the setting of Theorem~\ref{thm:main} and write
\[
    r:=M-1.
\]
For each arm $i$, write its reward sequence as
\begin{equation}
    X_{i,m}\sim\N(\mu_i,1),
    \qquad
    \hat\mu_{i,m}:=\frac1m\sum_{s=1}^m X_{i,s},
    \qquad m\ge1.
    \label{eq:reward-sequences}
\end{equation}
The variables $X_{i,m}$ are mutually independent over all pairs $(i,m)$ and are independent of the fresh random variables used by the algorithm to draw its actions. The $m$th pull of arm $i$ reveals $X_{i,m}$, so the pathwise identity $\hat\mu_i(t)=\hat\mu_{i,m}$ holds whenever $N_i(t)=m$. 

Independently of the adaptive sampling rule, the fixed-table average has
the unconditional distribution
$\hat\mu_{i,m}\sim\mathcal N(\mu_i,1/m)$. Let $\mathcal F_0$ be trivial and, for $t\ge1$, define
\[
    \mathcal F_t
    :=
    \sigma\!\left(I(s),R_s:1\le s\le t\right).
\]
Thus, $\mathcal F_{t-1}$ is simply the interaction history available before the action at round $t$ is chosen. We use the pre-decision convention: $N_i(t)$, $\hat\mu_i(t)$, the posterior $\Pi_{t-1}$, and the sampling distribution $\pi_t$ are all determined by this history. Since the algorithm then draws $I(t)$ from $\pi_t$ using fresh randomness independent of the history,
\begin{equation}
    \mathbb P\!\left(I(t)=i\mid\mathcal F_{t-1}\right)
    =
    \pi_{t,i}.
    \label{eq:conditional-action-law}
\end{equation}

After the initial pull of every arm, let
\[
    \mathcal T_T:=\{K+1,\ldots,T\}
\]
denote the remaining rounds. The regret decomposes exactly as
\begin{align}
    \Reg(T)
    ={}&
    \sum_{j=2}^K\Delta_j
    +
    \sum_{j=2}^K\Delta_j
    \sum_{t\in\mathcal T_T}
    \ind\{I(t)=j,\ \hat\mu_1(t)<\mu_1-\varepsilon\}
    \notag\\
    &+
    \sum_{j=2}^K\Delta_j
    \sum_{t\in\mathcal T_T}
    \ind\{I(t)=j,\ \hat\mu_1(t)\ge\mu_1-\varepsilon\}.
    \label{eq:generalM-regret-decomposition}
\end{align}
The first term is the initialization regret.  We bound the underestimation and non-underestimation terms in turn.

\subsection{Two EI tools}

We first record the Gaussian positive-part formula used throughout both regret proofs.

\begin{lemma}[Gaussian positive-part identity]
\label{lem:gaussian-positive-part}
If $X\sim\N(\delta,\sigma^2)$ with $\sigma>0$, then
\begin{equation}
    \E[X_+]
    =\sigma\phi(\delta/\sigma)+\delta\Phi(\delta/\sigma).
    \label{eq:gaussian-positive-part}
\end{equation}
\end{lemma}

\begin{proof}
Write $X=\delta+\sigma Z$ with $Z\sim\N(0,1)$. Direct integration gives
\begin{align*}
    \E[X_+]
    &=\int_{-\delta/\sigma}^{\infty}
      (\delta+\sigma z)\phi(z)\,\dd z\\
    &=\delta\int_{-\delta/\sigma}^{\infty}\phi(z)\,\dd z
      +\sigma\int_{-\delta/\sigma}^{\infty}z\phi(z)\,\dd z\\
    &=\delta\Phi(\delta/\sigma)
      +\sigma\phi(-\delta/\sigma)\\
    &=\delta\Phi(\delta/\sigma)
      +\sigma\phi(\delta/\sigma),
\end{align*}
where the third equality uses $\phi'(z)=-z\phi(z)$ and the last uses the symmetry of $\phi$.
\end{proof}

For distinct arms $i\ne j$, apply Lemma~\ref{lem:gaussian-positive-part} to $X=\theta_i-\theta_j\sim \N(\delta_{ij}(t),\sigma_{ij}^2(t))$ to obtain Equation~\eqref{eq:gij-closed-form}.

We begin with a pairwise Gaussian estimate that will be used to control the marginal EI. Fix a round and suppress its index. Write $m_i:=\hat\mu_i(t)$ and $v_i:=1/N_i(t)$, so
\[
    \theta_i\sim\N(m_i,v_i)
\]
independently across different arms. Reusing the definition from Equation~\eqref{eq:gij-closed-form}, abbreviate
\[
    G_{ij}:=G_{ij}(t)=\E[(\theta_i-\theta_j)_+],
    \qquad i\ne j.
\]
We set $G_{ii}:=0$ because repeated virtual selections use the same decision sample rather than independent copies.

\begin{lemma}[Pairwise Gaussian product]
\label{lem:generalM-pairwise-product}
For every $i,j\in[K]$,
\[
    G_{ij}G_{ji}\le v_i+v_j.
\]
\end{lemma}

\begin{proof}
The case $i=j$ is immediate.  For $i\ne j$, let
\[
    X:=\theta_i-\theta_j\sim\N(d,\sigma^2),
    \qquad
    \sigma^2:=v_i+v_j.
\]
After interchanging $i$ and $j$ if necessary, assume $d\ge0$, and set
\[
    a:=\E[X_+],
    \qquad
    b:=\E[(-X)_+].
\]
Since $a-b=\E[X]=d$, we have $a=d+b$.  If $0\le d\le\sigma$, monotonicity of the positive part gives
\[
    b\le \E[(-Z)_+]=\frac{\sigma}{\sqrt{2\pi}},
    \qquad Z\sim\N(0,\sigma^2),
\]
and hence
\[
    ab
    \le
    \left(\sigma+\frac{\sigma}{\sqrt{2\pi}}\right)
    \frac{\sigma}{\sqrt{2\pi}}
    <\sigma^2.
\]
If $d>\sigma$, write $X=d+Z$.  The scalar inequality
\[
    (y-d)_+\le \frac{y^2}{4d}
\]
and the symmetry of $Z\sim\N(0,\sigma^2)$ give
\[
    b=\E[(-d-Z)_+]
    =\E[(Z-d)_+]
    \le\frac{\E[Z^2]}{4d}
    =\frac{\sigma^2}{4d}.
\]
Therefore
\[
    ab
    \le
    \left(d+\frac{\sigma^2}{4d}\right)
    \frac{\sigma^2}{4d}
    \le\frac{5}{16}\sigma^2.
\]
Thus $G_{ij}G_{ji}=ab\le v_i+v_j$ in both cases.
\end{proof}

Let $\lambda_t$ denote the common marginal EI of the arms satisfying $\pi_{t,i}>0$, as given by Proposition~\ref{prop:ei-balance}. An arm with $\pi_{t,i}=0$ has marginal EI at most $\lambda_t$.

\begin{lemma}[Upper bound on the common marginal EI]
\label{lem:generalM-ei-information}
At every decision round,
\begin{equation}
    \lambda_t^2
    \le
    2\sum_{i=1}^K\frac{\pi_{t,i}}{N_i(t)}.
    \label{eq:generalM-ei-information}
\end{equation}
\end{lemma}

\begin{proof}
Fix $t$ and abbreviate $\pi:=\pi_t$.  Draw $\bm\theta\sim\Pi_{t-1}$ and, independently, sample $r-1=M-2$ arm indices $A_1,\ldots,A_{r-1}$ according to $\pi$. Write
\[
    W_{r-1}:=\max_{\ell\le r-1}\theta_{A_\ell},
\]
with $W_0:=-\infty$ when $M=2$. Define
\[
    R_j:=\E_{\bm\theta,A_{1:r-1}}
    [\max\{\theta_j,W_{r-1}\}],
\]
\[
    Q_{ij}:=\E_{\bm\theta,A_{1:r-1}}
    [\max\{\theta_i,\theta_j,W_{r-1}\}],
\]
and
\[
    H_{ij}:=H_{ij}(t)=Q_{ij}-R_j.
\]
By definition,
\[
    H_{ij}
    =
    \E_{\bm\theta,A_{1:r-1}}
    [(\theta_i-\max\{\theta_j,W_{r-1}\})_+].
\]
Conditioning on the last of the $r$ other virtual arm indices gives
\begin{equation}
    \bar s_{t,i}
    =\sum_{j=1}^K\pi_j
      \E_{\bm\theta,A_{1:r-1}}
      [\max\{\theta_i,\theta_j,W_{r-1}\}
       -\max\{\theta_j,W_{r-1}\}]
    =\sum_{j=1}^K\pi_jH_{ij}.
    \label{eq:generalM-contextual-row}
\end{equation}

The matrix $Q$ is symmetric.  Put
\[
    \overline R:=\sum_j\pi_jR_j.
\]
For every $i$ with $\pi_i>0$, Equation~\eqref{eq:generalM-contextual-row} and EI balance give
\[
    \sum_j\pi_jQ_{ij}
    =\lambda_t+\overline R
    =:\kappa.
\]
Let $I,J$ be independent draws from $\pi$ and set $Y:=Q_{IJ}$. Because $\pi_I>0$ and $\pi_J>0$ almost surely and $Q$ is symmetric,
\[
    \E[Y\mid I]=\E[Y\mid J]=\kappa.
\]
Consequently,
\begin{align*}
    \sum_{i,j}\pi_i\pi_jH_{ij}H_{ji}
    &=
    \E[(Y-R_J)(Y-R_I)]\\
    &=\E[Y^2]-\E[YR_I]-\E[YR_J]+\E[R_IR_J]\\
    &=\E[Y^2]-2\kappa\overline R+\overline R^2\\
    &=\operatorname{Var}(Y)+\kappa^2-2\kappa\overline R
      +\overline R^2\\
    &=\operatorname{Var}(Y)+(\kappa-\overline R)^2\\
    &=
    \operatorname{Var}(Y)+\lambda_t^2
    \ge\lambda_t^2.
\end{align*}
Here $\E[YR_I]=\E[R_I\E(Y\mid I)]=\kappa\overline R$, $\E[YR_J]=\kappa\overline R$, and independence of $I$ and $J$ gives $\E[R_IR_J]=\overline R^2$.

Moreover, $0\le H_{ij}\le G_{ij}$ because $\max\{\theta_j,W_{r-1}\}\ge\theta_j$.  Lemma \ref{lem:generalM-pairwise-product}, with $v_i=1/N_i(t)$, now gives
\begin{align*}
    \lambda_t^2
    &\le
    \sum_{i,j}\pi_i\pi_jG_{ij}G_{ji}\\
    &\le
    \sum_{i,j}\pi_i\pi_j
    \left(\frac1{N_i(t)}+\frac1{N_j(t)}\right)
    =
    2\sum_i\frac{\pi_{t,i}}{N_i(t)}.
\end{align*}
\end{proof}

\subsection{Bounding the underestimation term}
By Equation~\eqref{eq:reward-sequences}, whenever $N_1(t)=n$, $\hat\mu_1(t)=\hat\mu_{1,n}$ and $\hat\mu_{1,n}\sim\N(\mu_1,1/n)$.

For $n\ge1$ and $x\in\R$, define the threshold expected improvement
\begin{equation}
    \mathrm{EI}_{\varepsilon}(n,x)
    :=
    \E_{\theta\sim\N(x,1/n)}
    \bigl[(\theta-(\mu_1-\varepsilon))_+\bigr].
    \label{eq:generalM-threshold-call}
\end{equation}

\begin{lemma}[Gaussian threshold expected-improvement bound]
\label{lem:generalM-threshold-ei-bound}
Let $\phi$ and $\Phi$ denote the standard normal density and cumulative distribution function. For $y\ge0$, define
\[
    h(u):=\phi(u)-u\Phi(-u),
    \qquad u\ge0.
\]
Then
\begin{equation}
    \mathrm{EI}_{\varepsilon}(n,\mu_1-\varepsilon-y)
    =
    \frac1{\sqrt n}h(\sqrt n\,y),
    \label{eq:generalM-threshold-ei-exact}
\end{equation}
and there is a universal constant $a_1>0$ such that
\begin{equation}
    \mathrm{EI}_{\varepsilon}(n,\mu_1-\varepsilon-y)
    \ge
    \frac{a_1}{\sqrt n(1+ny^2)}
    \exp\!\left(-\frac{ny^2}{2}\right).
    \label{eq:generalM-threshold-ei-lower}
\end{equation}
Also, for every $x\le\mu_1-\varepsilon$,
\begin{equation}
    \mathrm{EI}_{\varepsilon}(n,x)
    \le\frac1{\sqrt{2\pi n}}.
    \label{eq:generalM-threshold-ei-upper}
\end{equation}
\end{lemma}

\begin{proof}
Write
\[
    \theta-(\mu_1-\varepsilon)
    =
    \frac{Z-\sqrt n\,y}{\sqrt n},
    \qquad Z\sim\N(0,1).
\]
Applying Lemma~\ref{lem:gaussian-positive-part} to $Z-u\sim\N(-u,1)$ with $u=\sqrt n\,y$ proves Equation~\eqref{eq:generalM-threshold-ei-exact}.  Moreover,
\[
    h(u)
    =\E_{Z\sim\N(0,1)}[(Z-u)_+]
    =\int_0^\infty v\phi(u+v)\,\dd v.
\]
Let $\delta_u:=(1+u)^{-1}$. For $0\le v\le\delta_u$, the ratio $\phi(u+v)/\phi(u)$ is at least $\exp(-3/2)$. Indeed,
\[
    \frac{\phi(u+v)}{\phi(u)}
    =\exp\!\left(-uv-\frac{v^2}{2}\right)
    \ge \exp(-3/2),
\]
because $uv\le u/(1+u)\le1$ and $v^2/2\le1/2$. Therefore
\begin{align*}
    h(u)
    &\ge
    \int_0^{\delta_u}v\phi(u+v)\,\dd v\\
    &\ge
    \exp(-3/2)\phi(u)\frac{\delta_u^2}{2}\\
    &\ge
    \frac{\exp(-3/2)}4
    \frac{\phi(u)}{1+u^2}.
\end{align*}
The last inequality uses $(1+u)^2\le2(1+u^2)$, so $\delta_u^2/2\ge1/[4(1+u^2)]$. Substituting $u=\sqrt n\,y$ proves \eqref{eq:generalM-threshold-ei-lower}.  Finally, the threshold expected improvement is increasing in its Gaussian mean, so for $x\le\mu_1-\varepsilon$ it is at most its value at $x=\mu_1-\varepsilon$, namely $1/\sqrt{2\pi n}$.
\end{proof}

For each round, put
\[
    n_t:=N_1(t),
    \qquad
    \mathrm{EI}_{\varepsilon,t}
    :=\mathrm{EI}_{\varepsilon}(n_t,\hat\mu_1(t)),
    \qquad
    \ell_r:=1+\log r,
\]
and define
\[
    \alpha_r:=\frac1{4r},
    \qquad
    \beta_r:=\frac{\ell_r}{4r},
    \qquad
    R_{e,t}:=\frac{8\ell_r}{\mathrm{EI}_{\varepsilon,t}^2}.
\]
The sufficiently sampled, optimistic, and remaining sufficiently sampled sets are
\begin{align}
    \mathcal S(t)
    &:=
    \left\{
      i\in[K]:
      N_i(t)\ge\max\{R_{e,t},n_t+1\}
    \right\},
    \label{eq:generalM-safe-set}\\
    \mathcal H(t)
    &:=
    \left\{
      i\in\mathcal S(t):
      i\ne1,\ \hat\mu_i(t)>\mu_1-\varepsilon
    \right\},\\
    \mathcal B(t)
    &:=
    \mathcal S(t)\setminus\mathcal H(t).
\end{align}
Write
\[
    \pi_t^{\mathcal S}:=\sum_{i\in\mathcal S(t)}\pi_{t,i},
    \qquad
    \pi_t^{\mathcal H}:=\sum_{i\in\mathcal H(t)}\pi_{t,i},
\]
and define the events
\[
    \mathcal X_u(t)
    :=
    \{\pi_t^{\mathcal H}\ge\alpha_r\},
    \qquad
    \mathcal Y_u(t)
    :=
    \{\pi_t^{\mathcal S}\ge1-\beta_r\}.
\]
Thus the underestimation rounds are split as
\begin{align}
    &\sum_{t\in\mathcal T_T}
    \ind\{\hat\mu_1(t)\le\mu_1-\varepsilon\}
    \notag\\
    &\quad=
    \underbrace{
      \sum_{t\in\mathcal T_T}
      \ind\{\hat\mu_1(t)\le\mu_1-\varepsilon,
             \ \mathcal X_u(t)\}
    }_{\mathrm{(L_{u,1})}}
    +
    \underbrace{
      \sum_{t\in\mathcal T_T}
      \ind\{\hat\mu_1(t)\le\mu_1-\varepsilon,
             \ \mathcal X_u^c(t)\}
    }_{\mathrm{(L_{u,2})}}.
    \label{eq:generalM-underestimation-split}
\end{align}

\subsubsection{Lower bound on the optimal-arm marginal expected improvement}

\begin{lemma}[Lower bound on $\bar s_{t,1}$]
\label{lem:generalM-threshold-ei}
On
\[
    \{N_1(t)=n,\ \hat\mu_{1,n}\le\mu_1-\varepsilon\}
    \cap\mathcal X_u^c(t)
    \cap\mathcal Y_u(t),
\]
one has
\begin{equation}
    \bar s_{t,1}
    \ge
    \frac{\mathrm{EI}_{\varepsilon}(n,\hat\mu_{1,n})}
         {2\mathrm e\sqrt r}.
    \label{eq:generalM-threshold-ei}
\end{equation}
\end{lemma}

\begin{proof}
During the block $N_1(t)=n$, write
\[
    \mathrm{EI}_{\varepsilon,n}
    :=\mathrm{EI}_{\varepsilon}(n,\hat\mu_{1,n}),
    \qquad
    R_e:=\frac{8\ell_r}{\mathrm{EI}_{\varepsilon,n}^2}.
\]
By Equation~\eqref{eq:generalM-threshold-ei-upper},
\begin{equation}
    R_e
    \ge16\pi\ell_r n
    >n+1.
    \label{eq:generalM-Re-dominates}
\end{equation}
Thus arm $1$ does not belong to $\mathcal S(t)$ and hence does not belong to $\mathcal B(t)$.

On $\mathcal X_u^c(t)\cap\mathcal Y_u(t)$,
\begin{equation}
    \sum_{i\in\mathcal B(t)}\pi_{t,i}
    >
    1-\alpha_r-\beta_r
    =
    1-\frac{2+\log r}{4r}
    =:b_r.
    \label{eq:generalM-B-mass}
\end{equation}
To lower-bound the marginal EI, consider the event that all $r$ virtual arms are sampled from $\mathcal B(t)$.  Conditional on their realized indices $a_1,\ldots,a_r\in\mathcal B(t)$, define
\[
    W:=\max_{\ell\le r}\theta_{a_\ell}.
\]
The selected arms may repeat. Since they all belong to $\mathcal B(t)$, each decision sample $\theta_{a_\ell}$ has mean at most $\mu_1-\varepsilon$ and variance at most $1/R_e$. Write these means and variances as $m_{a_\ell}$ and $v_{a_\ell}$. For every $\eta>0$, log-sum-exp, Jensen's inequality, and the Gaussian moment-generating function give
\begin{align*}
    \E_{\bm\theta\sim\Pi_{t-1}}
    [(W-(\mu_1-\varepsilon))_+]
    &\le
    \frac1\eta
    \E_{\bm\theta\sim\Pi_{t-1}}\log\!\left[
      1+\sum_{\ell=1}^r
      \exp\{\eta(\theta_{a_\ell}-(\mu_1-\varepsilon))\}
    \right]\\
    &\le
    \frac1\eta
    \log\E_{\bm\theta\sim\Pi_{t-1}}\!\left[
      1+\sum_{\ell=1}^r
      \exp\{\eta(\theta_{a_\ell}-(\mu_1-\varepsilon))\}
    \right]\\
    &=\frac1\eta
      \log\!\left[
        1+\sum_{\ell=1}^r
        \exp\!\left\{
          \eta(m_{a_\ell}-(\mu_1-\varepsilon))
          +\frac{\eta^2v_{a_\ell}}2
        \right\}
      \right]\\
    &\le\frac1\eta
      \log\!\left[1+r\exp\!\left(\frac{\eta^2}{2R_e}\right)\right]\\
    &\le
    \frac{\log(r+1)}{\eta}
    +\frac{\eta}{2R_e}.
\end{align*}
Here $m_{a_\ell}\le\mu_1-\varepsilon$ and $v_{a_\ell}\le1/R_e$, and the last step uses $1+r e^c\le(r+1)e^c$ for $c\ge0$. To minimize the right-hand side, define
\[
    f(\eta)
    :=
    \frac{\log(r+1)}{\eta}
    +
    \frac{\eta}{2R_e},
    \qquad \eta>0.
\]
Its derivative is
\[
    f'(\eta)
    =
    -\frac{\log(r+1)}{\eta^2}
    +
    \frac1{2R_e}.
\]
Since $f''(\eta)=2\log(r+1)/\eta^3>0$, the unique minimizer is
\[
    \eta^\star
    =
    \sqrt{2R_e\log(r+1)}.
\]
Substituting $\eta^\star$ gives
\begin{equation}
    \E_{\bm\theta\sim\Pi_{t-1}}
    [(W-(\mu_1-\varepsilon))_+]
    \le
    \sqrt{\frac{2\log(r+1)}{R_e}}
    \le\frac{\mathrm{EI}_{\varepsilon,n}}2.
    \label{eq:generalM-max-overshoot}
\end{equation}
The final inequality follows from $R_e=8\ell_r/\mathrm{EI}_{\varepsilon,n}^2$ and $\log(r+1)\le\ell_r$.

For all real $x,w$,
\begin{equation}
    (x-w)_+
    \ge
    (x-(\mu_1-\varepsilon))_+
    -(w-(\mu_1-\varepsilon))_+.
    \label{eq:generalM-positive-part-comparison}
\end{equation}
For each fixed index tuple, applying Equation~\eqref{eq:generalM-positive-part-comparison} with $x=\theta_1$ and $w=W$, taking expectations term by term, and using Equations~\eqref{eq:generalM-threshold-call} and \eqref{eq:generalM-max-overshoot} gives
\[
    \E_{\bm\theta\sim\Pi_{t-1}}[(\theta_1-W)_+]
    \ge\frac{\mathrm{EI}_{\varepsilon,n}}2.
\]
This bound holds for every tuple $a_1,\ldots,a_r\in\mathcal B(t)$. Let
\[
    \mathcal E_B
    :=
    \bigcap_{\ell=1}^r\{A_\ell\in\mathcal B(t)\}.
\]
Since the virtual indices are sampled independently according to $\pi_t$,
\[
    \mathbb{P}(\mathcal E_B)
    =
    \pi_t(\mathcal B(t))^r
    \ge b_r^r.
\]
Decomposing the marginal EI over $\mathcal E_B$ and its complement gives
\begin{align*}
    \bar s_{t,1}
    &=
    \E\!\left[\ind\{\mathcal E_B\}(\theta_1-W)_+\right]
    +
    \E\!\left[\ind\{\mathcal E_B^c\}(\theta_1-W)_+\right]\\
    &\ge
    \E\!\left[\ind\{\mathcal E_B\}(\theta_1-W)_+\right]\\
    &=
    \Prob(\mathcal E_B)
    \E\!\left[(\theta_1-W)_+\mid\mathcal E_B\right]\\
    &\ge
    \frac{b_r^r\mathrm{EI}_{\varepsilon,n}}{2}.
\end{align*}

It remains to bound this action-index probability.  Let
\[
    z_r:=\frac{2+\log r}{4r}.
\]
For every $r\ge1$, $0\le z_r\le1/2$, and $\log(1-z)\ge-2z$ on this interval.  Hence
\[
    b_r^r
    =(1-z_r)^r
    \ge\exp(-2rz_r)
    =\frac1{\mathrm e\sqrt r},
\]
which proves Equation~\eqref{eq:generalM-threshold-ei}.
\end{proof}

\subsubsection{Number of rounds in a fixed count block}
Let
\begin{equation}
    \sigma_n
    :=
    \inf\{t\ge K+1:N_1(t)=n\},
    \qquad
    \inf\varnothing:=\infty,
    \label{eq:best-count-times}
\end{equation}
\[
    \overline\sigma_n:=\min\{\sigma_n,T+1\},
    \qquad
    \mathcal A_n
    :=
    \{\sigma_n\le T,\ \hat\mu_{1,n}\le\mu_1-\varepsilon\}.
\]

We next bound the conditional expected number of non-optimistic rounds in the count-\(n\) block. Define the full count-$n$ non-optimistic block
\[
    L_n
    :=
    \sum_{t\in\mathcal T_T}
    \ind\{N_1(t)=n,\ \mathcal X_u^c(t)\}.
\]

\begin{lemma}[Fixed-count bound]
\label{lem:generalM-residence}
One has
\begin{equation}
    \ind\{\mathcal A_n\}
    \E\!\left[
      L_n\,\middle|\,
      \mathcal F_{\overline\sigma_n-1}
    \right]
    \le
    \ind\{\mathcal A_n\}
    \frac{128rK\log(T+1)}
         {\mathrm{EI}_{\varepsilon}(n,\hat\mu_{1,n})^2}.
    \label{eq:generalM-residence}
\end{equation}
\end{lemma}

\begin{proof}
Work on $\mathcal A_n$, and write
\[
    \mathrm{EI}_{\varepsilon,n}
    :=\mathrm{EI}_{\varepsilon}(n,\hat\mu_{1,n}),
    \qquad
    R_e:=\frac{8\ell_r}{\mathrm{EI}_{\varepsilon,n}^2}.
\]
Consider a round counted by $L_n$.  If $\mathcal Y_u^c(t)$ occurs, action mass greater than $\beta_r$ lies outside $\mathcal S(t)$.  Equation~\eqref{eq:generalM-Re-dominates} shows that, on this count block, $i\notin\mathcal S(t)$ implies $N_i(t)<R_e$.  Therefore
\begin{equation}
    \sum_i\frac{\pi_{t,i}}{N_i(t)}
    \ge
    \frac{\beta_r}{R_e}
    =\frac{\mathrm{EI}_{\varepsilon,n}^2}{32r}.
    \label{eq:generalM-progress-outside}
\end{equation}

If $\mathcal Y_u(t)$ occurs, Lemma~\ref{lem:generalM-threshold-ei} and EI balance imply
\[
    \lambda_t
    \ge\bar s_{t,1}
    \ge\frac{\mathrm{EI}_{\varepsilon,n}}
             {2\mathrm e\sqrt r}.
\]
Lemma~\ref{lem:generalM-ei-information} then gives
\begin{equation}
    \sum_i\frac{\pi_{t,i}}{N_i(t)}
    \ge\frac{\lambda_t^2}{2}
    \ge\frac{\mathrm{EI}_{\varepsilon,n}^2}{8\mathrm e^2r}
    \ge\frac{\mathrm{EI}_{\varepsilon,n}^2}{64r}.
    \label{eq:generalM-progress-inside}
\end{equation}
Thus every counted round satisfies the common lower bound
\begin{equation}
    \sum_i\frac{\pi_{t,i}}{N_i(t)}
    \ge\frac{\mathrm{EI}_{\varepsilon,n}^2}{64r}.
    \label{eq:generalM-common-progress}
\end{equation}

Use the log-count potential
\[
    \Psi_t:=\sum_{i=1}^K\log N_i(t).
\]
If arm $i$ is pulled at round $t$, then
\[
    \Psi_{t+1}-\Psi_t
    =\log\!\left(1+\frac1{N_i(t)}\right)
    \ge\frac1{2N_i(t)}.
\]
For a deterministic round \(t\), define the following indicator, which is determined before \(I(t)\) is drawn:
\[
    C_{t,n}
    :=
    \ind\{
      N_1(t)=n,\ 
      \hat\mu_1(t)=\hat\mu_{1,n}\le\mu_1-\varepsilon,\ 
      \mathcal X_u^c(t)
    \}.
\]
Using Equation~\eqref{eq:generalM-common-progress},
\begin{equation}
    \E\!\left[
      C_{t,n}(\Psi_{t+1}-\Psi_t)
      \,\middle|\,\mathcal F_{t-1}
    \right]
    \ge
    \frac{C_{t,n}}2
    \sum_i\frac{\pi_{t,i}}{N_i(t)}
    \ge
    \frac{\mathrm{EI}_{\varepsilon}(n,\hat\mu_{1,n})^2}
         {128r}C_{t,n}.
    \label{eq:generalM-one-step-drift}
\end{equation}

Fix a deterministic $s\in\mathcal T_T$ and define
\[
    \mathcal E_{n,s}
    :=
    \{\sigma_n=s,\ \hat\mu_{1,n}\le\mu_1-\varepsilon\}
    \in\mathcal F_{s-1}.
\]
Indeed, $\{\sigma_n=s\}$ is determined by the actions through round $s-1$, so $\sigma_n$ is a stopping time for the shifted filtration $(\mathcal F_{t-1})_{t\ge K+1}$. The empirical mean in $\mathcal E_{n,s}$ is also determined by those observations. On $\mathcal E_{n,s}$, $\mathrm{EI}_{\varepsilon}(n,\hat\mu_{1,n})$ is $\mathcal F_{s-1}$-measurable, and the definition of $C_{t,n}$ gives
\[
    \ind\{\mathcal E_{n,s}\}L_n
    =
    \ind\{\mathcal E_{n,s}\}
    \sum_{t=s}^T C_{t,n}.
\]
For every $t\ge s$, multiply Equation~\eqref{eq:generalM-one-step-drift} by $\ind\{\mathcal E_{n,s}\}$, condition down to $\mathcal F_{s-1}$, and sum over $t=s,\ldots,T$. The tower property gives
\begin{align*}
    &\ind\{\mathcal E_{n,s}\}
    \frac{\mathrm{EI}_{\varepsilon}(n,\hat\mu_{1,n})^2}{128r}
    \E[L_n\mid\mathcal F_{s-1}]\\
    &=\ind\{\mathcal E_{n,s}\}
    \frac{\mathrm{EI}_{\varepsilon}(n,\hat\mu_{1,n})^2}{128r}
    \E\!\left[\sum_{t=s}^T C_{t,n}\,\middle|\,\mathcal F_{s-1}\right]\\
    &\quad\le
    \ind\{\mathcal E_{n,s}\}
    \E\!\left[
      \sum_{t=s}^T
      C_{t,n}(\Psi_{t+1}-\Psi_t)
      \,\middle|\,\mathcal F_{s-1}
    \right]\\
    &\quad\le
    \ind\{\mathcal E_{n,s}\}
    \E[\Psi_{T+1}-\Psi_s\mid\mathcal F_{s-1}]\\
    &\quad\le
    \ind\{\mathcal E_{n,s}\}K\log(T+1).
\end{align*}
The events $\{\mathcal E_{n,s}:s\in\mathcal T_T\}$ are disjoint and their union is $\mathcal A_n$. They also show directly that $\mathcal A_n\in\mathcal F_{\overline\sigma_n-1}$. Decomposing the conditional expectation over the disjoint events \(\{E_{n,s}:s\in\mathcal T_T\}\) gives,
\begin{equation}
    \ind\{\mathcal A_n\}
    \E[L_n\mid\mathcal F_{\overline\sigma_n-1}]
    =
    \sum_{s\in\mathcal T_T}
      \ind\{\mathcal E_{n,s}\}
      \E[L_n\mid\mathcal F_{s-1}].
    \label{eq:generalM-stopped-block-identity}
\end{equation}
Summing the preceding deterministic-$s$ inequalities and applying this identity proves Equation~\eqref{eq:generalM-residence}.
\end{proof}

\subsubsection{First part of the underestimation split}

\begin{proposition}[Bound for rounds satisfying $\mathcal X_u^c(t)$]
\label{prop:generalM-under-main}
There is a universal constant $C>0$ such that
\begin{equation}
    \E\left[
    \sum_{t\in\mathcal T_T}
    \ind\{
      \hat\mu_1(t)\le\mu_1-\varepsilon,\ 
      \mathcal X_u^c(t)
    \}\right]
    \le
    C\sqrt{rKT\log(T+1)}
    \frac{\log(2/\varepsilon)}{\varepsilon^3}.
    \label{eq:generalM-under-main}
\end{equation}
\end{proposition}

\begin{proof}
Set
\[
    B_{T,M}:=128rK\log(T+1).
\]
The left-hand side of Equation~\eqref{eq:generalM-under-main} is
\[
    \sum_{n\ge1}\E[\ind\{\mathcal A_n\}L_n].
\]
Lemma~\ref{lem:generalM-residence}, together with $L_n\le T$, gives on $\mathcal A_n$
\begin{equation}
    \E\!\left[
      L_n\,\middle|\,
      \mathcal F_{\overline\sigma_n-1}
    \right]
    \le
    \min\!\left\{
      T,
      \frac{B_{T,M}}
           {\mathrm{EI}_{\varepsilon}(n,\hat\mu_{1,n})^2}
    \right\}
    \le
    \frac{\sqrt{B_{T,M}T}}
         {\mathrm{EI}_{\varepsilon}(n,\hat\mu_{1,n})}.
    \label{eq:generalM-sqrt-interpolation}
\end{equation}

Since $\mathcal A_n\in\mathcal F_{\overline\sigma_n-1}$, the tower property and Equation~\eqref{eq:generalM-sqrt-interpolation} give
\begin{align*}
    \E[\ind\{\mathcal A_n\}L_n]
    &=
    \E\!\left[
      \ind\{\mathcal A_n\}
      \E\!\left[
        L_n\,\middle|\,
        \mathcal F_{\overline\sigma_n-1}
      \right]
    \right]\\
    &\le
    \sqrt{B_{T,M}T}\,
    \E\!\left[
      \frac{\ind\{\mathcal A_n\}}
           {\mathrm{EI}_{\varepsilon}(n,\hat\mu_{1,n})}
    \right]\\
    &\le
    \sqrt{B_{T,M}T}\,
    \E\!\left[
      \frac{\ind\{\hat\mu_{1,n}\le\mu_1-\varepsilon\}}
           {\mathrm{EI}_{\varepsilon}(n,\hat\mu_{1,n})}
    \right],
\end{align*}
where the last inequality uses
\[
    \mathcal A_n
    \subseteq
    \{\hat\mu_{1,n}\le\mu_1-\varepsilon\}.
\]
Summing over $n$ gives
\begin{equation}
    \sum_{n\ge1}
    \E[\ind\{\mathcal A_n\}L_n]
    \le
    \sqrt{B_{T,M}T}
    \sum_{n\ge1}
    \E\!\left[
      \frac{\ind\{\hat\mu_{1,n}\le\mu_1-\varepsilon\}}
           {\mathrm{EI}_{\varepsilon}(n,\hat\mu_{1,n})}
    \right].
    \label{eq:generalM-integral-start}
\end{equation}

Since $\hat\mu_{1,n}\sim\N(\mu_1,1/n)$, substitute
\[
    \hat\mu_{1,n}=\mu_1-\varepsilon-y,
    \qquad y\ge0.
\]
The fixed-$n$ expectation in Equation~\eqref{eq:generalM-integral-start} equals
\[
    \int_0^\infty
    \frac{f_{\hat\mu_{1,n}}(\mu_1-\varepsilon-y)}
         {\mathrm{EI}_{\varepsilon}(n,\mu_1-\varepsilon-y)}
    \,\dd y.
\]
Here
\begin{align*}
    f_{\hat\mu_{1,n}}(\mu_1-\varepsilon-y)
    &=
    \sqrt{\frac n{2\pi}}
    \exp\!\left(-\frac n2(\varepsilon+y)^2\right)\\
    &=
    \sqrt{\frac n{2\pi}}
    \exp\!\left(-\frac{n\varepsilon^2}{2}
                 -n\varepsilon y
                 -\frac{ny^2}{2}\right).
\end{align*}
Using the preceding density formula and Equation~\eqref{eq:generalM-threshold-ei-lower}, for each $n\ge1$,
\[
    \sqrt{B_{T,M}T}\,
    \E\!\left[
      \frac{\ind\{\hat\mu_{1,n}\le\mu_1-\varepsilon\}}
           {\mathrm{EI}_{\varepsilon}(n,\hat\mu_{1,n})}
    \right]
    \le
    C\sqrt{B_{T,M}T}\,
    n\exp(-n\varepsilon^2/2)
    \int_0^\infty
      (1+ny^2)\exp(-n\varepsilon y)\,\dd y.
\]
Using the standard Gamma-integral identity
\[
    \int_0^\infty y^k e^{-ay}\,\dd y
    =
    \frac{k!}{a^{k+1}},
    \qquad a>0,\quad k\in\{0,1,2,\ldots\},
\]
with $a=n\varepsilon$ and $k\in\{0,2\}$, we obtain
\begin{equation}
    \int_0^\infty(1+ny^2)e^{-n\varepsilon y}\,\dd y
    =
    \int_0^\infty e^{-n\varepsilon y}\,\dd y
    +
    n\int_0^\infty y^2e^{-n\varepsilon y}\,\dd y
    =
    \frac1{n\varepsilon}
    +\frac2{n^2\varepsilon^3}.
\end{equation}
For each $n\ge1$, the term indexed by $n$ on the right-hand side of Equation~\eqref{eq:generalM-integral-start} is bounded by
\[
    C\sqrt{B_{T,M}T}\,
    \exp(-n\varepsilon^2/2)
    \left(
      \frac1\varepsilon+
      \frac2{n\varepsilon^3}
    \right).
\]
Therefore, summing this bound over $n$ gives
\begin{align*}
    \sum_{n\ge1}\E[\ind\{\mathcal A_n\}L_n]
    \le
    C\sqrt{B_{T,M}T}
    \left[
      \frac1\varepsilon
      \sum_{n\ge1}\exp(-n\varepsilon^2/2)
      +
      \frac2{\varepsilon^3}
      \sum_{n\ge1}
      \frac{\exp(-n\varepsilon^2/2)}n
    \right].
\end{align*}
For $\varepsilon\le1$, the two series satisfy
\[
    \sum_{n\ge1}\exp(-n\varepsilon^2/2)
    =
    \frac1{\exp(\varepsilon^2/2)-1}
    \le
    \frac2{\varepsilon^2},
\]
and
\[
    \sum_{n\ge1}
    \frac{\exp(-n\varepsilon^2/2)}n
    =
    -\log\!\left(1-\exp(-\varepsilon^2/2)\right)
    \le
    C\log(2/\varepsilon).
\]
Substituting these bounds into the preceding inequality yields
\begin{align*}
    \sum_{n\ge1}\E[\ind\{\mathcal A_n\}L_n]
    &\le
    C\sqrt{B_{T,M}T}
    \left[
      \frac1{\varepsilon^3}
      +
      \frac{\log(2/\varepsilon)}{\varepsilon^3}
    \right]\\
    &\le
    C\sqrt{B_{T,M}T}
    \frac{\log(2/\varepsilon)}{\varepsilon^3}.
\end{align*}
Finally, substituting $B_{T,M}=128rK\log(T+1)$ and absorbing the numerical factor into $C$ gives
\[
    \sum_{n\ge1}\E[\ind\{\mathcal A_n\}L_n]
    \le
    C\sqrt{rKT\log(T+1)}
    \frac{\log(2/\varepsilon)}{\varepsilon^3},
\]
which is Equation~\eqref{eq:generalM-under-main}.
\end{proof}

\subsubsection{Second part of the underestimation split}

\begin{proposition}[Bound for rounds satisfying $\mathcal X_u(t)$]
\label{prop:generalM-under-optimistic}
One has
\begin{equation}
    \E\left[
    \sum_{t\in\mathcal T_T}
    \ind\{
      \hat\mu_1(t)\le\mu_1-\varepsilon,\ 
      \mathcal X_u(t)
    \}\right]
    \le
    4r\sum_{i=2}^K
    \frac1{\exp((\Delta_i-\varepsilon)^2/2)-1}.
    \label{eq:generalM-under-optimistic}
\end{equation}
\end{proposition}

\begin{proof}
On $\mathcal X_u(t)$, action mass at least
$\alpha_r=1/(4r)$ lies on $\mathcal H(t)$, and hence
\begin{equation}
    \ind\{\mathcal X_u(t)\}
    \le
    \frac{1}{\alpha_r}
    \sum_{i=2}^K
    \pi_{t,i}\ind\{i\in\mathcal H(t)\}.
    \label{eq:optimistic-mass-to-actions}
\end{equation}
All quantities in Equation \eqref{eq:optimistic-mass-to-actions}, as well as the indicator \(\mathbf{1}\{\widehat{\mu}_1(t)\le \mu_1-\varepsilon\}\), are determined before \(I(t)\) is drawn. Multiplying Equation \eqref{eq:optimistic-mass-to-actions} by this indicator, summing over \(t\), and using Equation \eqref{eq:conditional-action-law} together with the tower property yields
\begin{equation}
    \E
    \sum_{t\in\mathcal T_T}
    \ind\{
      \hat\mu_1(t)\le\mu_1-\varepsilon,\ 
      \mathcal X_u(t)
    \}
    \le
    4r\sum_{i=2}^K
    \E\sum_{t\in\mathcal T_T}
    \ind\{I(t)=i,\ \hat\mu_i(t)>\mu_1-\varepsilon\}.
\end{equation}
For each fixed $i\ne1$ and $m\ge1$, arm $i$ can be pulled at most once while its pre-count equals $m$.  On that round, $\hat\mu_i(t)=\hat\mu_{i,m}$, so the pull is included in the preceding sum only if
\[
    \{\hat\mu_{i,m}>\mu_1-\varepsilon\}.
\]
Because $\hat\mu_{i,m}\sim\N(\mu_i,1/m)$,
\[
    \Prob(\hat\mu_{i,m}>\mu_1-\varepsilon)
    \le
    \exp\!\left(-\frac{m(\Delta_i-\varepsilon)^2}{2}\right).
\]
Therefore, for each $i\ne1$,
\[
    \E\sum_{t\in\mathcal T_T}
    \ind\{I(t)=i,\ \hat\mu_i(t)>\mu_1-\varepsilon\}
    \le
    \sum_{m\ge1}
    \exp\!\left(-\frac{m(\Delta_i-\varepsilon)^2}{2}\right)
    =
    \frac1{\exp((\Delta_i-\varepsilon)^2/2)-1}.
\]
Substituting this bound into the preceding inequality and summing over $i=2,\ldots,K$ proves Equation~\eqref{eq:generalM-under-optimistic}.
\end{proof}

Combining the preceding two propositions bounds the total number of rounds on which the optimal empirical mean is underestimated.

\begin{corollary}[Total number of underestimation rounds]
\label{cor:generalM-under-total}
\begin{align}
    \E
    \sum_{t\in\mathcal T_T}
    \ind\{\hat\mu_1(t)\le\mu_1-\varepsilon\}
    \le{}&
    C\sqrt{rKT\log(T+1)}
    \frac{\log(2/\varepsilon)}{\varepsilon^3}
    \notag\\
    &+
    4r\sum_{i=2}^K
    \frac1{\exp((\Delta_i-\varepsilon)^2/2)-1}.
    \label{eq:generalM-under-total}
\end{align}
\end{corollary}
\begin{proof}
Split the underestimation event according to $\mathcal X_u(t)$ and $\mathcal X_u^c(t)$.  Applying Propositions~\ref{prop:generalM-under-main} and \ref{prop:generalM-under-optimistic} to the two resulting terms and adding their bounds gives the result.
\end{proof}

\subsection{Bounding the non-underestimation term}

Fix a suboptimal arm $j\ge2$.  Define the following pre-decision event
\begin{equation}
    \mathcal X_s(t)
    :=
    \{\hat\mu_1(t)\ge\mu_1-\varepsilon,\ 
      \hat\mu_j(t)\le\mu_j+\varepsilon\}.
    \label{eq:generalM-good-comparison}
\end{equation}
On this event,
\begin{equation}
    d_t
    :=
    \hat\mu_1(t)-\hat\mu_j(t)
    \ge
    \Delta_j-2\varepsilon
    \ge\frac35\Delta_j
    >0,
    \label{eq:generalM-empirical-gap}
\end{equation}
where the second inequality uses $\varepsilon\le\Delta_{\min}/5\le\Delta_j/5$.

\subsubsection{A CDF representation of marginal EI}

For a fixed round $t$ and posterior realization $\bm\theta$, define
\[
    F_{\pi_t}(u;\bm\theta)
    :=
    \sum_{k=1}^K\pi_{t,k}\ind\{\theta_k\le u\}.
\]
Conditional on $\bm\theta$, the maximum $W_r=W_{M-1}$ of the $r$ other virtual actions satisfies
\[
    \Prob(W_r\le u\mid\bm\theta)
    =F_{\pi_t}(u;\bm\theta)^r.
\]

\begin{lemma}[Marginal-EI CDF identity]
\label{lem:generalM-cdf-identity}
For every arm $i$,
\begin{equation}
    \bar s_{t,i}
    =
    \E_{\bm\theta\sim\Pi_{t-1}}
    \int_{-\infty}^{\theta_i}
    F_{\pi_t}(u;\bm\theta)^r\,\dd u.
    \label{eq:generalM-cdf-ei}
\end{equation}
Consequently,
\begin{equation}
    \bar s_{t,1}-\bar s_{t,j}
    =
    \E_{\bm\theta\sim\Pi_{t-1}}
    \int_{\theta_j}^{\theta_1}
    F_{\pi_t}(u;\bm\theta)^r\,\dd u.
    \label{eq:generalM-cdf-difference}
\end{equation}
\end{lemma}

\begin{proof}
Let
\[
    W_r:=\max_{\ell\le r}\theta_{A_\ell},
    \qquad
    A_1,\ldots,A_r\overset{\mathrm{i.i.d.}}{\sim}\pi_t.
\]
Conditional on $\bm\theta$, the event $\{W_r\le u\}$ occurs if and only if every sampled arm satisfies $\theta_{A_\ell}\le u$.  Hence, by independence of the virtual action indices,
\[
    \mathbb P(W_r\le u\mid\bm\theta)
    =
    \prod_{\ell=1}^r
    \mathbb P(\theta_{A_\ell}\le u\mid\bm\theta)
    =
    F_{\pi_t}(u;\bm\theta)^r.
\]

For all real $x,w$,
\[
    (x-w)_+
    =
    \int_{-\infty}^x
    \ind\{w\le u\}\,\dd u.
\]
Applying this identity with $x=\theta_i$ and $w=W_r$, and then using Tonelli's theorem to interchange the nonnegative integral and expectation, gives
\begin{align*}
    \bar s_{t,i}
    &=
    \E_{\bm\theta,A_{1:r}}
    [(\theta_i-W_r)_+]\\
    &=
    \E_{\bm\theta}
    \int_{-\infty}^{\theta_i}
    \mathbb P(W_r\le u\mid\bm\theta)\,\dd u\\
    &=
    \E_{\bm\theta\sim\Pi_{t-1}}
    \int_{-\infty}^{\theta_i}
    F_{\pi_t}(u;\bm\theta)^r\,\dd u,
\end{align*}
which proves Equation~\eqref{eq:generalM-cdf-ei}.

Applying this identity to arms $1$ and $j$ and subtracting yields
\begin{align*}
    \bar s_{t,1}-\bar s_{t,j}
    &=
    \E_{\bm\theta\sim\Pi_{t-1}}
    \left[
      \int_{-\infty}^{\theta_1}
      F_{\pi_t}(u;\bm\theta)^r\,\dd u
      -
      \int_{-\infty}^{\theta_j}
      F_{\pi_t}(u;\bm\theta)^r\,\dd u
    \right]\\
    &=
    \E_{\bm\theta\sim\Pi_{t-1}}
    \int_{\theta_j}^{\theta_1}
    F_{\pi_t}(u;\bm\theta)^r\,\dd u,
\end{align*}
which proves Equation~\eqref{eq:generalM-cdf-difference}.
\end{proof}

\subsubsection{The probability of a saturated arm}

For this fixed arm $j$, define
\begin{equation}
    A_{j,M}
    :=
    \left(\frac4{\sqrt\pi\,\Delta_j}\right)^{1/r},
    \qquad
    f_t^\star
    :=
    \frac{d_t^2}{2(1/N_1(t)+1/N_j(t))}.
    \label{eq:generalM-A-and-f}
\end{equation}

\begin{lemma}[Pairwise saturation inequality]
\label{lem:generalM-saturation-probability}
If $\pi_{t,j}>0$, then $\pi_{t,j}<1$ and
\begin{equation}
    \left(\frac{\pi_{t,j}}{1-\pi_{t,j}}\right)^r
    \le
    \frac{G_{j1}(t)}{G_{1j}(t)}.
    \label{eq:generalM-weight-ratio}
\end{equation}
On $\mathcal X_s(t)$, for every $j$,
\begin{equation}
    \pi_{t,j}
    \le
    A_{j,M}\exp\!\left(-\frac{f_t^\star}{r}\right).
    \label{eq:generalM-pi-j}
\end{equation}
\end{lemma}

\begin{proof}
Suppress the round index $t$.  If $\theta_1>\theta_j$, then for every $u\in[\theta_j,\theta_1)$, arm $j$ contributes its probability mass $\pi_j$ to $F_\pi(u;\bm\theta)$.  Hence
\[
    \int_{\theta_j}^{\theta_1}
    F_\pi(u;\bm\theta)^r\,\dd u
    \ge
    \pi_j^r(\theta_1-\theta_j).
\]
If $\theta_j>\theta_1$, then for every $u\in[\theta_1,\theta_j)$, arm $j$ does not contribute to $F_\pi(u;\bm\theta)$, so $F_\pi(u;\bm\theta)\le1-\pi_j$.  Since the integral is taken in the reverse direction,
\[
    \int_{\theta_j}^{\theta_1}
    F_\pi(u;\bm\theta)^r\,\dd u
    \ge
    -(1-\pi_j)^r(\theta_j-\theta_1).
\]
Combining the two cases gives the pointwise bound
\[
    \int_{\theta_j}^{\theta_1}
    F_\pi(u;\bm\theta)^r\,\dd u
    \ge
    \pi_j^r(\theta_1-\theta_j)_+
    -(1-\pi_j)^r(\theta_j-\theta_1)_+.
\]
Taking posterior expectations and using the definitions of $G_{1j}$ and $G_{j1}$ in Equation~\eqref{eq:generalM-cdf-difference} yields
\[
    \bar s_{t,1}-\bar s_{t,j}
    \ge
    \pi_j^rG_{1j}
    -(1-\pi_j)^rG_{j1}.
\]
If $\pi_{t,j}>0$, Proposition~\ref{prop:ei-balance} gives $\bar s_{t,j}\ge\bar s_{t,1}$, and hence
\[
    \pi_j^rG_{1j}(t)
    \le
    (1-\pi_j)^rG_{j1}(t).
\]
The Gaussian laws make both pairwise EIs strictly positive, so this inequality rules out $\pi_j=1$.  Division proves Equation~\eqref{eq:generalM-weight-ratio}.

On $\mathcal X_s(t)$, Jensen's inequality for the convex positive part and Equation~\eqref{eq:generalM-empirical-gap} give
\begin{equation}
    G_{1j}(t)
    \ge d_t
    \ge\frac35\Delta_j
    \ge\frac{\Delta_j}{4}.
    \label{eq:generalM-forward-ei}
\end{equation}
Put
\[
    q_t:=\frac1{N_1(t)}+\frac1{N_j(t)}.
\]
Since $\theta_j-\theta_1\sim\N(-d_t,q_t)$, Lemma~\ref{lem:gaussian-positive-part} gives
\begin{align}
    G_{j1}(t)
    &=
    \sqrt{q_t}\,
    \phi\!\left(\frac{d_t}{\sqrt{q_t}}\right)
    -
    d_t\,
    \Phi\!\left(-\frac{d_t}{\sqrt{q_t}}\right)
    \notag\\
    &\le
    \sqrt{q_t}\,
    \phi\!\left(\frac{d_t}{\sqrt{q_t}}\right)
    \notag\\
    &=
    \frac{\sqrt{q_t}}{\sqrt{2\pi}}
    \exp(-f_t^\star)
    \le
    \frac1{\sqrt\pi}\exp(-f_t^\star),
    \label{eq:generalM-reverse-ei}
\end{align}
where the final inequality uses $q_t\le2$. Combining Equations~\eqref{eq:generalM-weight-ratio}, \eqref{eq:generalM-forward-ei}, and \eqref{eq:generalM-reverse-ei}, then taking an $r$-th root, yields
\[
    \pi_{t,j}
    \le
    \frac{\pi_{t,j}}{1-\pi_{t,j}}
    \le
    \left(\frac4{\sqrt\pi\,\Delta_j}\right)^{1/r}
    \exp\!\left(-\frac{f_t^\star}{r}\right).
\]
If $\pi_{t,j}=0$, the same conclusion is immediate.
\end{proof}

\subsubsection{Bounding suboptimal-arm pulls via the best arm}

\begin{lemma}[Best-arm probability dominance]
\label{lem:generalM-best-arm-probability-dominance}
Suppose
\[
    \hat\mu_1(t)>\hat\mu_j(t),
    \qquad
    N_1(t)\le N_j(t).
\]
Then
\begin{equation}
    \pi_{t,j}\le\pi_{t,1}.
    \label{eq:generalM-best-arm-probability-dominance}
\end{equation}
\end{lemma}

\begin{proof}
Suppress the round index $t$ and suppose, for contradiction, that $\pi_j>\pi_1$.  Conditional on a fixed realization of all arms other than $1$ and $j$, define
\[
    B(u)
    :=
    \sum_{k\notin\{1,j\}}
    \pi_k\ind\{\theta_k\le u\},
    \qquad
    H_B(x)
    :=
    \int_0^x(B(u)+\pi_1)^r\,\dd u.
\]
Since $B(u)$ is nonnegative and nondecreasing, $H_B$ is increasing and convex.

For $u$ between $\theta_1$ and $\theta_j$, one has $F_\pi(u;\bm\theta)=B(u)+\pi_j$ when $\theta_1>\theta_j$, and $F_\pi(u;\bm\theta)=B(u)+\pi_1$ when $\theta_j>\theta_1$.  Moreover,
\[
    H_B(\theta_1)-H_B(\theta_j)
    =
    \int_{\theta_j}^{\theta_1}
    (B(u)+\pi_1)^r\,\dd u.
\]
Therefore, Equation~\eqref{eq:generalM-cdf-difference} gives
\begin{equation}
\begin{aligned}
    \bar s_{t,1}-\bar s_{t,j}
    ={}&
    \E[H_B(\theta_1)-H_B(\theta_j)]\\
    &+
    \E\!\left[
      \ind\{\theta_1>\theta_j\}
      \int_{\theta_j}^{\theta_1}
      \bigl\{
        (B(u)+\pi_j)^r-(B(u)+\pi_1)^r
      \bigr\}\,\dd u
    \right].
\end{aligned}
\label{eq:generalM-best-arm-comparison-decomposition}
\end{equation}

Let $\delta:=\hat\mu_1-\hat\mu_j>0$.  Since $1/N_1\ge1/N_j$, we may write
\[
    \theta_1
    \overset{d}{=}
    \theta_j'+\delta+Z,
\]
where $\theta_j'$ has the same distribution as $\theta_j$, and $Z$ is a centered Gaussian with variance $1/N_1-1/N_j$. On the auxiliary probability space, take $\theta_j'$ and $Z$ independent of each other and of the already fixed realization of all other arms. For every fixed function $B$ and every fixed value $x$ of $\theta_j'$, conditional Jensen and the monotonicity of $H_B$ give
\[
    \E_Z[H_B(x+\delta+Z)]
    \ge H_B(x+\delta+\E Z)
    =H_B(x+\delta)
    \ge H_B(x).
\]
Integrating this inequality over $\theta_j'$ yields
\[
    \E[H_B(\theta_1)]
    \ge\E[H_B(\theta_j'+\delta)]
    \ge\E[H_B(\theta_j')]
    =\E[H_B(\theta_j)].
\]
Hence the first term in Equation~\eqref{eq:generalM-best-arm-comparison-decomposition} is nonnegative.

On $\{\theta_1>\theta_j\}$, the second integrand is strictly positive because $\pi_j>\pi_1$ and $r\ge1$.  Since $\mathbb P(\theta_1>\theta_j)>0$, the second term is strictly positive.  Thus
\[
    \bar s_{t,1}>\bar s_{t,j}.
\]
However, $\pi_j>0$, so EI balance requires $\bar s_{t,j}\ge\bar s_{t,1}$, a contradiction.  Therefore $\pi_j\le\pi_1$.
\end{proof}

\subsubsection{Counting saturated pulls}

We first isolate overestimation of the suboptimal arm.

\begin{lemma}[Arm $j$ overestimation bound]
\label{lem_armj_overestimated}
\begin{equation}
    \E\left[\sum_{t\in\mathcal T_T}
    \ind\{
      I(t)=j,\ 
      \hat\mu_1(t)\ge\mu_1-\varepsilon,\ 
      \hat\mu_j(t)>\mu_j+\varepsilon
    \}\right]
    \le\frac{2}{\varepsilon^2}.
    \label{eq:generalM-arm-j-overestimate}
\end{equation}
\end{lemma}

\begin{proof}
Drop the condition on arm $1$.  For each $m\ge1$, arm $j$ can be pulled at most once while its pre-count equals $m$.  On that round, $\hat\mu_j(t)=\hat\mu_{j,m}$, so the pull is counted only if $\hat\mu_{j,m}>\mu_j+\varepsilon$.  Therefore the expectation is at most
\[
    \sum_{m\ge1}
    \Prob(\hat\mu_{j,m}>\mu_j+\varepsilon)
    \le
    \sum_{m\ge1}\exp(-m\varepsilon^2/2)
    \le\frac{2}{\varepsilon^2}.
\]
\end{proof}

For the saturation argument, define
\begin{equation}
    L_{j,M}:=1+2TA_{j,M},
    \qquad
    T_{j,M}
    :=
    \left\lceil
      \max\!\left\{
        1,\frac{r}{D_j^{\varepsilon}}\log L_{j,M}
      \right\}
    \right\rceil.
    \label{eq:generalM-saturation-threshold}
\end{equation}
Also define
\[
    \mathcal Y_s(t):=\{N_j(t)\ge3T_{j,M}\},
    \qquad
    \mathcal Z_s(t):=\{N_1(t)\ge3T_{j,M}\}.
\]

\begin{lemma}[Saturation count]
\label{lem:generalM-saturation-count}
For every suboptimal arm $j$,
\begin{equation}
    \E\left[\sum_{t\in\mathcal T_T}
    \ind\{\hat\mu_1(t)\ge\mu_1-\varepsilon,\ I(t)=j\}\right]\le
    C\left[
      \frac1{\varepsilon^2}
      +
      \frac{r}{D_j^{\varepsilon}}\log L_{j,M}
    \right]
    \label{eq:generalM-saturation-count}
\end{equation}
for a universal constant $C>0$.
\end{lemma}

\begin{proof}
Lemma~\ref{lem_armj_overestimated} removes the rounds on which arm $j$ is overestimated at a cost at most $2/\varepsilon^2$.  It remains to count pulls on $\mathcal X_s(t)$.

Before $N_j(t)$ reaches $3T_{j,M}$, arm $j$ can be drawn at most $3T_{j,M}$ times.  The remaining pulls split as
\begin{align*}
    &\sum_{t\in\mathcal T_T}
    \ind\{I(t)=j,\ \mathcal X_s(t),\ \mathcal Y_s(t)\}\\
    &\quad=
    \underbrace{
      \sum_{t\in\mathcal T_T}
      \ind\{I(t)=j,\ \mathcal X_s(t),
             \ \mathcal Y_s(t),\ \mathcal Z_s(t)\}
    }_{\mathrm{(L_{s,1})}}
    +
    \underbrace{
      \sum_{t\in\mathcal T_T}
      \ind\{I(t)=j,\ \mathcal X_s(t),
             \ \mathcal Y_s(t),\ \mathcal Z_s^c(t)\}
    }_{\mathrm{(L_{s,2})}}.
\end{align*}

For the rounds counted by $\mathrm{(L_{s,2})}$, $N_1(t)<N_j(t)$ and $d_t>0$.  Lemma~\ref{lem:generalM-best-arm-probability-dominance} therefore gives $\pi_{t,j}\le\pi_{t,1}$.  Since \(\mathcal X_s(t)\), \(\mathcal Y_s(t)\), and \(\mathcal Z_s^c(t)\) are determined by \(\mathcal F_{t-1}\), Equation~\eqref{eq:conditional-action-law} gives
\begin{align*}
    \E[\mathrm{(L_{s,2})}]
    &\le
    \E\sum_{t\in\mathcal T_T}
    \ind\{\mathcal X_s(t),\mathcal Y_s(t),\mathcal Z_s^c(t)\}
    \pi_{t,1}\\
    &=
    \E\sum_{t\in\mathcal T_T}
    \ind\{\mathcal X_s(t),\mathcal Y_s(t),\mathcal Z_s^c(t),
           \ I(t)=1\}\\
    &\le3T_{j,M}+1,
\end{align*}
because arm $1$ can be pulled only that many times before its count reaches $3T_{j,M}$.

For rounds counted by $\mathrm{(L_{s,1})}$, both counts are at least $3T_{j,M}$, and hence
\[
    \frac1{1/N_1(t)+1/N_j(t)}
    \ge\frac32T_{j,M}.
\]
Equations~\eqref{eq:generalM-empirical-gap} and \eqref{eq:generalM-saturation-threshold} give
\begin{align*}
    \frac{f_t^\star}{r}
    &=\frac{d_t^2}{2r}
      \frac1{1/N_1(t)+1/N_j(t)}\\
    &\ge
      \frac{3T_{j,M}}{4r}(\Delta_j-2\varepsilon)^2\\
    &\ge
      \frac34
      \frac{(\Delta_j-2\varepsilon)^2}{D_j^\varepsilon}
      \log L_{j,M}\\
    &=q_j\log L_{j,M},\\
    q_j
    &:=
    \frac32
    \left(
      \frac{\Delta_j-2\varepsilon}
           {\Delta_j-\tfrac32\varepsilon}
    \right)^2.
\end{align*}
Since $\varepsilon\le\Delta_j/5$,
\[
    q_j
    \ge
    \frac32\left(\frac{3/5}{7/10}\right)^2
    =\frac{54}{49}>1.
\]
Lemma~\ref{lem:generalM-saturation-probability} now yields
\[
    \pi_{t,j}
    \le
    A_{j,M}L_{j,M}^{-q_j}
    \le
    \frac1{2T},
\]
where the last inequality uses $L_{j,M}\ge2TA_{j,M}$ and $q_j>1$. Therefore
\[
    \E[\mathrm{(L_{s,1})}]
    \le
    \sum_{t\in\mathcal T_T}\frac1{2T}
    \le\frac12.
\]
Combining all four contributions explicitly gives
\begin{align*}
    &\E\left[\sum_{t\in\mathcal T_T}
      \ind\{\hat\mu_1(t)\ge\mu_1-\varepsilon,\ I(t)=j\}\right]\\
    &\quad\le
      \frac2{\varepsilon^2}
      +3T_{j,M}+(3T_{j,M}+1)+\frac12\\
    &\quad\le
      C\left[
        \frac1{\varepsilon^2}
        +\frac{r}{D_j^\varepsilon}\log L_{j,M}
      \right].
\end{align*}
The last inequality uses $\varepsilon\le1$ and $T_{j,M}\le2+r\log(L_{j,M})/D_j^\varepsilon$. This proves Equation~\eqref{eq:generalM-saturation-count}.
\end{proof}

\subsection{Completion of the proof of Theorem~\ref{thm:main}}

\begin{proof}[Proof of Theorem~\ref{thm:main}]
Take expectations in the regret decomposition~\eqref{eq:generalM-regret-decomposition}.  We bound its underestimation and non-underestimation terms separately.

Since at most one suboptimal arm is pulled at each round and $\Delta_j\le\Delta_{\max}$, Corollary~\ref{cor:generalM-under-total} gives
\begin{align*}
    &\E
    \sum_{j=2}^K\Delta_j
    \sum_{t\in\mathcal T_T}
    \ind\{I(t)=j,\ \hat\mu_1(t)<\mu_1-\varepsilon\}\\
    &\quad\le
    \Delta_{\max}
    \E\sum_{t\in\mathcal T_T}
    \ind\{\hat\mu_1(t)\le\mu_1-\varepsilon\}\\
    &\quad\le
    C\Delta_{\max}
    \sqrt{rKT\log(T+1)}
    \frac{\log(2/\varepsilon)}{\varepsilon^3}
    +
    4r\Delta_{\max}
    \sum_{i=2}^K
    \frac1{\exp((\Delta_i-\varepsilon)^2/2)-1}.
\end{align*}

Applying Lemma~\ref{lem:generalM-saturation-count} to each suboptimal arm $j$ gives
\begin{align*}
    &\E
    \sum_{j=2}^K\Delta_j
    \sum_{t\in\mathcal T_T}
    \ind\{I(t)=j,\ \hat\mu_1(t)\ge\mu_1-\varepsilon\}\\
    &\quad\le
    C\sum_{j=2}^K\Delta_j
    \left[
      \frac1{\varepsilon^2}
      +
      \frac{r}{D_j^\varepsilon}\log L_{j,M}
    \right].
\end{align*}

Finally, the initialization regret is $\sum_{j=2}^K\Delta_j$.  Adding the three bounds and substituting
\[
    L_{j,M}
    =
    1+2T
    \left(\frac4{\sqrt\pi\,\Delta_j}\right)^{1/r}
\]
gives Equation~\eqref{eq:generalM-main-bound}, completing the proof.
\end{proof}

\section{Proof of Theorem~\ref{thm:asymptotic-regret}}
\label{app:asymptotic-proof}

We retain the notation and the pre-decision convention of Appendix~\ref{app:full_proof}. In particular, $r=M-1$, $N_i(t)$ and $\hat\mu_i(t)$ are evaluated before the action at round $t$, and $\hat\mu_{i,m}$ is the average of the first $m$ rewards of arm $i$. We also use the quantities $G_{ij}(t)$, $H_{ij}(t)$, $\bar s_{t,i}$, and $\lambda_t$ defined in Section~\ref{sec:algorithm} and Appendix~\ref{app:full_proof}.

All constants in this appendix may depend on the fixed mean vector, $K$, $M$, and a displayed positive gap. They never depend on the round, the horizon, or the current sample counts. A constant denoted by $C$ may have a different value in different displays.

\subsection{Reward sequences and simultaneous deviation bounds}

The reward sequences $X_{i,m}$ and their averages $\hat\mu_{i,m}$ were defined in Equation~\eqref{eq:reward-sequences}. Fix $\beta>2$ and define the two nonnegative random variables
\begin{align}
    B_{\mathrm{dev}}
    &:={}
    \max_{i\in[K]}\sup_{m\ge1}
    \left\{
      m(\hat\mu_{i,m}-\mu_i)^2-\beta\log(em)
    \right\}_+,
    \label{eq:asymptotic-deviation-all}\\
    B_{\mathrm{dev},-1}
    &:={}
    \max_{j>1}\sup_{m\ge1}
    \left\{
      m(\hat\mu_{j,m}-\mu_j)^2-\beta\log(em)
    \right\}_+.
    \label{eq:asymptotic-deviation-suboptimal}
\end{align}
The definition gives, simultaneously for every arm and every count,
\begin{align}
    |\hat\mu_{i,m}-\mu_i|
    &\le
    \sqrt{\frac{B_{\mathrm{dev}}+\beta\log(em)}{m}},
    \label{eq:asymptotic-deviation-all-bound}\\
    |\hat\mu_{j,m}-\mu_j|
    &\le
    \sqrt{\frac{B_{\mathrm{dev},-1}+\beta\log(em)}{m}},
    \qquad j>1.
    \label{eq:asymptotic-deviation-suboptimal-bound}
\end{align}

For $x\ge0$, a union bound and the Gaussian tail inequality give
\begin{align}
    \Prob(B_{\mathrm{dev}}>x)
    &\le
    \sum_{i=1}^K\sum_{m=1}^{\infty}
    \Prob\!\left(
      |\hat\mu_{i,m}-\mu_i|
      >\sqrt{\frac{x+\beta\log(em)}{m}}
    \right)\\
    &\le
    2K\sum_{m=1}^{\infty}
    \exp\!\left(-\frac{x+\beta\log(em)}2\right)\\
    &=
    2K e^{-x/2}\sum_{m=1}^{\infty}(em)^{-\beta/2},
    \label{eq:asymptotic-deviation-tail}
\end{align}
which is exponentially small in $x$.
The series is finite because $\beta>2$. Therefore $B_{\mathrm{dev}}$ and $B_{\mathrm{dev},-1}$ are finite almost surely, have moments of every polynomial order, and satisfy
\begin{equation}
    \E[e^{aB_{\mathrm{dev}}}]<\infty,
    \qquad
    \E[e^{aB_{\mathrm{dev},-1}}]<\infty,
    \qquad 0<a<\frac12.
    \label{eq:asymptotic-deviation-moments}
\end{equation}
Moreover, $B_{\mathrm{dev},-1}$ is independent of the entire reward sequence of arm $1$.

For the proof only, reveal all reward sequences at time zero and define
\begin{equation}
    \widetilde{\mathcal F}_t
    :=
    \sigma\!\left(
      \mathcal F_t,
      X_{i,m}:i\in[K],\ m\ge1
    \right).
    \label{eq:asymptotic-enlarged-filtration}
\end{equation}
This does not change the algorithm. Its fresh random variable for drawing $I(t)$ is independent of the reward sequences, so
\begin{equation}
    \Prob(I(t)=i\mid\widetilde{\mathcal F}_{t-1})
    =\pi_{t,i}.
    \label{eq:asymptotic-action-law}
\end{equation}
let $A_{1:M}\sim\pi$ denote $M$ i.i.d. \textit{virtual arm} indices drawn from $\pi$.
Likewise, the Gaussian variables and the virtual arm indices used in the definition of $\bar s_{t,i}$ can be generated from fresh random variables. Thus, conditional on $\widetilde{\mathcal F}_{t-1}$, the $r=M-1$ virtual arm indices are independent draws from $\pi_t$.

Recall the count times $\sigma_n$ from Equation~\eqref{eq:best-count-times}. Under the pre-decision convention, $\sigma_1=K+1$, and the rounds
\begin{equation}
    t=\sigma_n,\ldots,\sigma_{n+1}-1
    \label{eq:asymptotic-count-interval}
\end{equation}
are exactly the rounds on which $N_1(t)=n$ and $\hat\mu_1(t)=\hat\mu_{1,n}$. We will prove below that every $\sigma_n$ is finite almost surely. For conditioning at the start of this interval, write
\begin{equation}
    \mathcal G_n:=\widetilde{\mathcal F}_{\sigma_n-1}.
    \label{eq:asymptotic-stopped-field}
\end{equation}
Thus, \(\mathcal G_n\) is the information available immediately before the random round \(\sigma_n\). The interval estimates below are used only under the induction hypothesis $\sigma_n<\infty$ almost surely; Proposition~\ref{prop:asymptotic-finite-underestimation} starts from $\sigma_1=K+1$ and verifies this hypothesis successively.

\subsection{Pairwise expected improvement between suboptimal arms}

The only additional assumption used in this subsection is that distinct suboptimal arms have different means, which is an additional assumption in Theorem \ref{thm:asymptotic-regret}. Define
\begin{equation}
    \gamma_{-}
    :=
    \begin{cases}
      \displaystyle
      \min_{\substack{i,j>1\\i\ne j}}|\mu_i-\mu_j|,
      &K\ge3,\\[0.6em]
      1,&K=2,
    \end{cases}
    \label{eq:asymptotic-suboptimal-gap}
\end{equation}
which is larger than $0$.
For $K=2$, every assertion below involving two distinct suboptimal arms is empty.

\begin{lemma}[Pairwise EI between well-sampled suboptimal arms]
\label{lem:asymptotic-pairwise-suboptimal-ei}
There are constants $A_\gamma,C_\gamma>0$ such that, for every round $t$ and all distinct $i,j>1$, with $m_{ij}(t):=\min\{N_i(t),N_j(t)\}$,
\begin{equation}
    \min\{G_{ij}(t),G_{ji}(t)\}
    \le
    A_\gamma
    \exp(C_\gamma B_{\mathrm{dev},-1})
    \exp\!\left(-\frac{\gamma_-^2m_{ij}(t)}{16}\right).
    \label{eq:asymptotic-pairwise-suboptimal-ei}
\end{equation}
\end{lemma}

\begin{proof}
Fix $t$ and distinct $i,j>1$. Put
\[
    N_i:=N_i(t),\qquad N_j:=N_j(t),\qquad
    m:=\min\{N_i,N_j\},
\]
and define
\begin{equation}
    v_{ij}:=\frac1{N_i}+\frac1{N_j},
    \qquad
    w_{ij}:=\frac1{v_{ij}}
    =\frac{N_iN_j}{N_i+N_j}.
    \label{eq:asymptotic-pairwise-variance}
\end{equation}
Then
\begin{equation}
    \frac m2\le w_{ij}\le m.
    \label{eq:asymptotic-harmonic-count-bounds}
\end{equation}

Write $e_k:=\hat\mu_k(t)-\mu_k$. Suppose first that $N_i\ge N_j=m$. Equation~\eqref{eq:asymptotic-deviation-suboptimal-bound} and $w_{ij}/N_i\le m/N_i$ give
\begin{align}
    w_{ij}e_i^2
    &\le
    \frac{m}{N_i}
    \{B_{\mathrm{dev},-1}+\beta\log(eN_i)\}\\
    &\le
    B_{\mathrm{dev},-1}
    +\beta\log(em)
    +\beta\frac{m}{N_i}\log(N_i/m)\\
    &\le
    B_{\mathrm{dev},-1}+\beta\log(em)+\frac\beta e,
    \label{eq:asymptotic-weighted-error-large-count}
\end{align}
because $\sup_{x\ge1}(\log x)/x=1/e$. Also,
\begin{equation}
    w_{ij}e_j^2
    \le B_{\mathrm{dev},-1}+\beta\log(em).
    \label{eq:asymptotic-weighted-error-small-count}
\end{equation}
The same two bounds hold after interchanging $i$ and $j$. Hence
\begin{align}
    w_{ij}(e_i-e_j)^2
    &\le2w_{ij}e_i^2+2w_{ij}e_j^2\\
    &\le C\{1+B_{\mathrm{dev},-1}+\log(em)\}.
    \label{eq:asymptotic-weighted-error-difference}
\end{align}

Let $\delta:=\mu_i-\mu_j$. Since $(a+b)^2\ge a^2/2-b^2$ for all real $a,b$, Equations \eqref{eq:asymptotic-harmonic-count-bounds} and \eqref{eq:asymptotic-weighted-error-difference} imply
\begin{align}
    w_{ij}(\hat\mu_i(t)-\hat\mu_j(t))^2
    &\ge
    \frac{w_{ij}\delta^2}{2}
    -w_{ij}(e_i-e_j)^2\\
    &\ge
    \frac{m\gamma_-^2}{4}
    -C\{1+B_{\mathrm{dev},-1}+\log(em)\}.
    \label{eq:asymptotic-empirical-pairwise-gap}
\end{align}

Set
\[
    x:=
    \sqrt{w_{ij}}\,
    |\hat\mu_i(t)-\hat\mu_j(t)|.
\]
The posterior difference has variance $v_{ij}=w_{ij}^{-1}$. Lemma~\ref{lem:gaussian-positive-part} therefore gives
\begin{equation}
    \min\{G_{ij}(t),G_{ji}(t)\}
    =
    \sqrt{v_{ij}}\{\phi(x)-x\Phi(-x)\}
    \le
    \sqrt{v_{ij}}\phi(x).
    \label{eq:asymptotic-smaller-pairwise-ei}
\end{equation}
Since $v_{ij}\le2$ and $\phi(x)=(2\pi)^{-1/2}e^{-x^2/2}$, Equation~\eqref{eq:asymptotic-empirical-pairwise-gap} yields
\begin{equation}
    \min\{G_{ij}(t),G_{ji}(t)\}
    \le
    C\exp(CB_{\mathrm{dev},-1})m^C
    \exp\!\left(-\frac{\gamma_-^2m}{8}\right).
    \label{eq:asymptotic-pairwise-before-absorption}
\end{equation}
Finally, for $a:=\gamma_-^2/16$,
\begin{equation}
    \sup_{m\ge1}m^Ce^{-am}
    \le
    \max\!\left\{1,\left(\frac{C}{ae}\right)^C\right\}.
    \label{eq:asymptotic-polynomial-absorption}
\end{equation}
Applying this inequality to one half of the exponential factor in Equation~\eqref{eq:asymptotic-pairwise-before-absorption} proves Equation~\eqref{eq:asymptotic-pairwise-suboptimal-ei}.
\end{proof}

\subsection{Rounds on which the optimal empirical mean is too small}

Fix a deterministic number
\begin{equation}
    \max_{j>1}\mu_j<b<\mu_1,
    \qquad
    \varepsilon_b:=\mu_1-b,
    \qquad
    d_b:=b-\max_{j>1}\mu_j.
    \label{eq:asymptotic-separator}
\end{equation}
Both $\varepsilon_b$ and $d_b$ are positive.

Fix $n$ for which $\sigma_n<\infty$ almost surely; this condition will be verified inductively in Proposition~\ref{prop:asymptotic-finite-underestimation}. Suppose that $\hat\mu_{1,n}\le b$ and write
\begin{equation}
    \hat\mu_{1,n}=b-y,
    \qquad y\ge0.
    \label{eq:asymptotic-underestimation-depth}
\end{equation}
Using the threshold expected improvement from Equation~\eqref{eq:generalM-threshold-call} with $\varepsilon=\varepsilon_b$, define
\begin{equation}
    e_n(y)
    :=\mathrm{EI}_{\varepsilon_b}(n,b-y),
    \qquad
    E_n(y):=2^{-(r+2)}e_n(y).
    \label{eq:asymptotic-threshold-ei}
\end{equation}
Equations~\eqref{eq:generalM-threshold-ei-exact}-- \eqref{eq:generalM-threshold-ei-lower} give
\begin{align}
    e_n(y)
    &=\frac1{\sqrt n}
      \{\phi(\sqrt n\,y)-\sqrt n\,y\Phi(-\sqrt n\,y)\},
    \label{eq:asymptotic-threshold-ei-exact}\\
    e_n(y)
    &\ge
    \frac{a_1}{\sqrt n(1+ny^2)}
    \exp\!\left(-\frac{ny^2}{2}\right).
    \label{eq:asymptotic-threshold-ei-lower}
\end{align}
Also $0<E_n(y)<1$.

Choose a constant $C_*>0$, depending only on the fixed instance, $K$, $M$, and $\beta$, large enough for the estimates below, and define
\begin{equation}
    u_n(y)
    :=
    \max\!\left\{
      n+1,
      \left\lceil
        C_*\left(
          1+B_{\mathrm{dev},-1}+\log\frac{8K}{E_n(y)}
        \right)
      \right\rceil
    \right\}.
    \label{eq:asymptotic-count-threshold}
\end{equation}

We next verify all consequences required from this threshold. The elementary inequality
\begin{equation}
    \log(em)\le2\sqrt m,
    \qquad m\ge1,
    \label{eq:asymptotic-log-sqrt}
\end{equation}
implies that, whenever $m\ge C_*(1+B_{\mathrm{dev},-1})$,
\begin{equation}
    \frac{B_{\mathrm{dev},-1}+\beta\log(em)}m
    \le\frac1{C_*}+\frac{2\beta}{\sqrt{C_*}}.
    \label{eq:asymptotic-threshold-error-bound}
\end{equation}
Choose $C_*$ so that the right-hand side is at most $d_b^2/4$. Then every suboptimal arm with $N_j(t)\ge u_n(y)$ satisfies
\begin{equation}
    \hat\mu_j(t)\le\mu_j+d_b/2\le b-d_b/2.
    \label{eq:asymptotic-suboptimal-below-separator}
\end{equation}
Consequently, for its posterior sample,
\begin{align}
    \Prob(\theta_j>b\mid\widetilde{\mathcal F}_{t-1})
    &\le
    \exp\!\left(-\frac{N_j(t)d_b^2}{8}\right).
    \label{eq:asymptotic-suboptimal-posterior-tail}
\end{align}

Let
\[
    c_*:=\min\{d_b^2/8,\gamma_-^2/16\}.
\]
By increasing $C_*$ further so that
\begin{equation}
    c_*C_*\ge C_\gamma+1,
    \qquad
    A_\gamma e^{-c_*C_*}\le1,
    \qquad
    c_*C_*\ge1,
    \qquad
    e^{-d_b^2C_*/8}\le\frac1{4r},
    \label{eq:asymptotic-Cstar-conditions}
\end{equation}
Equations~\eqref{eq:asymptotic-pairwise-suboptimal-ei}, \eqref{eq:asymptotic-suboptimal-posterior-tail}, and \eqref{eq:asymptotic-count-threshold} imply, for $N_i(t),N_j(t)\ge u_n(y)$ and distinct $i,j>1$,
\begin{align}
    \min\{G_{ij}(t),G_{ji}(t)\}
    &\le\frac{E_n(y)}{8K},
    \label{eq:asymptotic-small-pairwise-ei}\\
    \Prob(\theta_j>b\mid\widetilde{\mathcal F}_{t-1})
    &\le\frac1{4r}.
    \label{eq:asymptotic-small-posterior-tail}
\end{align}
For example, the first inequality follows from
\begin{align*}
    &A_\gamma e^{C_\gamma B_{\mathrm{dev},-1}}
      e^{-c_*u_n(y)}\\
    &\quad\le
      A_\gamma e^{-c_*C_*}
      e^{-(c_*C_*-C_\gamma)B_{\mathrm{dev},-1}}
      \left(\frac{E_n(y)}{8K}\right)^{c_*C_*}
    \le\frac{E_n(y)}{8K}.
\end{align*}
The posterior-tail inequality follows in the same way, without the factor $A_\gamma e^{C_\gamma B_{\mathrm{dev},-1}}$; increasing $C_*$ by a fixed amount makes its right-hand side at most $1/(4r)$.

For a round in the interval with $N_1(t)=n$, define the set
\begin{equation}
    \mathcal S_n(t)
    :=\{j>1:N_j(t)\ge u_n(y)\}.
    \label{eq:asymptotic-well-sampled-set}
\end{equation}
If this set contains an arm with positive sampling probability, choose
\begin{equation}
    \ell_t\in
    \argmin\{\hat\mu_i(t):i\in\mathcal S_n(t),\ \pi_{t,i}>0\}.
    \label{eq:asymptotic-minimum-mean-arm}
\end{equation}
Then, for every $j\in\mathcal S_n(t)$ with $\pi_{t,j}>0$,
\begin{equation}
    \hat\mu_{\ell_t}(t)\le\hat\mu_j(t),
    \qquad
    G_{\ell_tj}(t)
    =\min\{G_{\ell_tj}(t),G_{j\ell_t}(t)\}.
    \label{eq:asymptotic-selected-row-direction}
\end{equation}
The terms with $\pi_{t,j}=0$ vanish. Equation~\eqref{eq:asymptotic-small-pairwise-ei} therefore gives
\begin{equation}
    \sum_{j\in\mathcal S_n(t)}
    \pi_{t,j}G_{\ell_tj}(t)
    \le\frac{E_n(y)}8.
    \label{eq:asymptotic-small-row-sum}
\end{equation}

We also need a bound that applies when either arm has not reached the threshold. For a posterior variable $\theta_i\sim\N(\hat\mu_i(t),1/N_i(t))$,
\begin{equation}
    \E[|\theta_i|\mid\widetilde{\mathcal F}_{t-1}]
    \le|\hat\mu_i(t)|+1.
    \label{eq:asymptotic-posterior-absolute-moment}
\end{equation}
For $j>1$, Equations~\eqref{eq:asymptotic-deviation-suboptimal-bound} and $\log(em)/m\le1$ give
\[
    |\hat\mu_j(t)|
    \le|\mu_j|+\sqrt{B_{\mathrm{dev},-1}}+\sqrt\beta.
\]
For arm $1$ in the present interval, $|\hat\mu_1(t)|=|b-y|\le|b|+y$. Hence, for every $i,j$,
\begin{align}
    G_{ij}(t)
    &\le\E[|\theta_i|+|\theta_j|\mid
          \widetilde{\mathcal F}_{t-1}]\\
    &\le U_n(y,B_{\mathrm{dev},-1})
    :=C\{1+y+\sqrt{1+B_{\mathrm{dev},-1}}\}.
    \label{eq:asymptotic-all-pairwise-bound}
\end{align}

The purpose of the bound $G_{ij}(t)\le U_n(y,B_{\mathrm{dev},-1})$ is to turn a lower bound on marginal EI into a lower bound on sampling probability. To see this, let $q_t:=1-\pi_t(\mathcal S_n(t))$ be the probability of selecting arm $1$ or an insufficiently sampled suboptimal arm. Whenever $\ell_t$ (defined in \eqref{eq:asymptotic-minimum-mean-arm}) exists, the marginal EI is bounded by the weighted pairwise EI sum, so Equation~\eqref{eq:asymptotic-small-row-sum} and the bound just obtained give
\[
    \bar s_{t,\ell_t}
    \le\sum_{j=1}^K\pi_{t,j}G_{\ell_tj}(t)
    \le\frac{E_n(y)}8+U_n(y,B_{\mathrm{dev},-1})q_t.
\]
If $q_t\le1/2$, the next lemma gives $\bar s_{t,1}\ge2E_n(y)$, while EI balance gives $\bar s_{t,\ell_t}\ge\bar s_{t,1}$. The displayed upper bound then forces $q_t\ge E_n(y)/U_n(y,B_{\mathrm{dev},-1})$. If $q_t>1/2$, a lower bound is already available. Thus the uniform upper bound on pairwise EI ensures a positive probability of either ending the current block by selecting arm $1$, or bringing an insufficiently sampled arm closer to its threshold. We formalize this probability bound in Equation~\eqref{eq:asymptotic-progress-probability} and apply Lemma~\ref{lem:asymptotic-counting-interval} with $u=u_n(y)$ and $p=p_n(y,B_{\mathrm{dev},-1})$ to obtain the expected block-length bound in Equation~\eqref{eq:asymptotic-underestimated-interval-bound}.

\begin{lemma}[Lower bound on $\bar s_{t,1}$]
\label{lem:asymptotic-best-ei-lower}
If $\pi_t(\mathcal S_n(t))\ge1/2$, then
\begin{equation}
    \bar s_{t,1}\ge2E_n(y).
    \label{eq:asymptotic-best-ei-lower}
\end{equation}
\end{lemma}

\begin{proof}
Let $A_1,\ldots,A_r$ be the other $r=M-1$ virtual arm indices in the definition of $\bar s_{t,1}$. Conditional on $\widetilde{\mathcal F}_{t-1}$,
\begin{equation}
    \Prob(A_1,\ldots,A_r\in\mathcal S_n(t))
    =\pi_t(\mathcal S_n(t))^r\ge2^{-r}.
    \label{eq:asymptotic-virtual-indices-in-set}
\end{equation}
Fix any realized indices $a_1,\ldots,a_r$ in this set. A union bound and Equation~\eqref{eq:asymptotic-small-posterior-tail} yield
\begin{equation}
    \Prob\!\left(\max_{1\le q\le r}\theta_{a_q}\le b
      \mid\widetilde{\mathcal F}_{t-1},a_1,\ldots,a_r\right)
    \ge1-r\frac1{4r}=\frac34.
    \label{eq:asymptotic-virtual-values-below-b}
\end{equation}
Repeated indices use the same posterior coordinate, so the union bound remains valid. Arm $1$ is independent of the posterior coordinates of the suboptimal arms. On the event in Equation~\eqref{eq:asymptotic-virtual-values-below-b},
\[
    \left(\theta_1-\max_{1\le q\le r}\theta_{a_q}\right)_+
    \ge(\theta_1-b)_+.
\]
Using Equations~\eqref{eq:asymptotic-virtual-indices-in-set}, \eqref{eq:asymptotic-virtual-values-below-b}, and \eqref{eq:asymptotic-threshold-ei}, we obtain
\begin{align*}
    \bar s_{t,1}
    &\ge2^{-r}\frac34 e_n(y)\\
    &=3\cdot2^{-r-2}e_n(y)
    =3E_n(y)
    \ge2E_n(y).
\end{align*}
\end{proof}

Combining the preceding estimates gives a lower bound on the probability of sampling arm $1$ or a suboptimal arm whose count is below $u_n(y)$. If $\pi_t(\mathcal S_n(t))<1/2$, then this probability is larger than $1/2$. If $\pi_t(\mathcal S_n(t))\ge1/2$ and the set contains an arm with positive sampling probability, use $\ell_t$ from Equation~\eqref{eq:asymptotic-minimum-mean-arm}. The EI balance condition, Equations~\eqref{eq:asymptotic-best-ei-lower}, \eqref{eq:asymptotic-small-row-sum}, and \eqref{eq:asymptotic-all-pairwise-bound} give
\begin{align}
    2E_n(y)
    &\le\bar s_{t,1}
    \le\bar s_{t,\ell_t}\\
    &\le
    \sum_{j\in\mathcal S_n(t)}
      \pi_{t,j}G_{\ell_tj}(t)
    +U_n(y,B_{\mathrm{dev},-1})
      \left(\pi_{t,1}
      +\sum_{\substack{j>1\\N_j(t)<u_n(y)}}\pi_{t,j}\right)\\
    &\le\frac{E_n(y)}8
    +U_n(y,B_{\mathrm{dev},-1})
      \left(\pi_{t,1}
      +\sum_{\substack{j>1\\N_j(t)<u_n(y)}}\pi_{t,j}\right).
    \label{eq:asymptotic-ei-progress-calculation}
\end{align}
If $\mathcal S_n(t)$ contains no arm with positive sampling probability, the probability in parentheses is one. Thus, in every case,
\begin{equation}
    \pi_{t,1}
    +\sum_{\substack{j>1\\N_j(t)<u_n(y)}}\pi_{t,j}
    \ge
    p_n(y,B_{\mathrm{dev},-1})
    :=\min\!\left\{
      \frac12,\frac{E_n(y)}{U_n(y,B_{\mathrm{dev},-1})}
    \right\}.
    \label{eq:asymptotic-progress-probability}
\end{equation}

We now prove the counting inequality used here and later.

\begin{lemma}[Counting rounds until the next pull of arm $1$]
\label{lem:asymptotic-counting-interval}
Fix an interval on which $N_1(t)=n$. Suppose there are an integrable, $\mathcal G_n$-measurable positive integer $u$ and a $\mathcal G_n$-measurable number $p>0$ such that, at every round in the interval,
\begin{equation}
    \pi_{t,1}
    +\sum_{\substack{j>1\\N_j(t)<u}}\pi_{t,j}\ge p.
    \label{eq:asymptotic-counting-assumption}
\end{equation}
Then
\begin{equation}
    \E[\sigma_{n+1}-\sigma_n\mid\mathcal G_n]
    \le\frac{1+(K-1)u}{p}.
    \label{eq:asymptotic-counting-conclusion}
\end{equation}
\end{lemma}

\begin{proof}
For every deterministic $t\ge K+1$, define
\begin{equation}
    V_t
    :=
    \ind\{\sigma_n\le t<\sigma_{n+1}\}
    \left[1+\sum_{j=2}^K(u-N_j(t))_+\right].
    \label{eq:asymptotic-counting-potential}
\end{equation}
At the start of the interval, $0\le V_{\sigma_n}\le1+(K-1)u$. If $t<\sigma_{n+1}$ and arm $1$ is sampled, then $t+1=\sigma_{n+1}$ and $V_{t+1}=0$, so $V_{t+1}-V_t\le-1$. If a suboptimal arm with $N_j(t)<u$ is sampled, its summand decreases by one. Every other action leaves the bracketed quantity unchanged. Therefore
\begin{equation}
    \ind\{\sigma_n\le t<\sigma_{n+1}\}(V_{t+1}-V_t)
    \le
    -\ind\{\sigma_n\le t<\sigma_{n+1}\}
     \ind\{I(t)=1\text{ or }[I(t)>1,N_{I(t)}(t)<u]\}.
    \label{eq:asymptotic-counting-one-step}
\end{equation}
On $\{\sigma_n\le t\}$, the $\mathcal G_n$-measurable variables $u$ and $p$ are also $\widetilde{\mathcal F}_{t-1}$-measurable. Using Equation~\eqref{eq:asymptotic-action-law} and then Equation~\eqref{eq:asymptotic-counting-assumption},
\begin{equation}
    \E\!\left[
      \ind\{\sigma_n\le t<\sigma_{n+1}\}(V_{t+1}-V_t)
      \middle|\widetilde{\mathcal F}_{t-1}
    \right]
    \le-p\ind\{\sigma_n\le t<\sigma_{n+1}\}.
    \label{eq:asymptotic-counting-drift}
\end{equation}

For a deterministic $Q\ge1$, set $\rho_Q:=\min\{\sigma_{n+1},\sigma_n+Q\}$. The random-limit sum can be written with deterministic indices as
\begin{equation}
    V_{\rho_Q}-V_{\sigma_n}
    =\sum_{t=K+1}^{\infty}
      \ind\{\sigma_n\le t<\rho_Q\}(V_{t+1}-V_t).
    \label{eq:asymptotic-counting-deterministic-sum}
\end{equation}
Apply the tower property to every term, first with the sum truncated at a deterministic upper limit and then by conditional dominated convergence. For fixed $Q$, the absolute value of the sum is at most $Q\{1+(K-1)u\}$, which is finite and $\mathcal G_n$-measurable. The indicator $\ind\{\sigma_n\le t<\rho_Q\}$ is $\widetilde{\mathcal F}_{t-1}$-measurable. Since $p$ is $\mathcal G_n$-measurable, Equation~\eqref{eq:asymptotic-counting-drift} gives
\begin{align}
    \E[V_{\rho_Q}-V_{\sigma_n}\mid\mathcal G_n]
    &\le
    -p\,
    \E\!\left[
      \sum_{t=K+1}^{\infty}
      \ind\{\sigma_n\le t<\rho_Q\}
      \middle|\mathcal G_n
    \right]\\
    &=-p\,\E[\rho_Q-\sigma_n\mid\mathcal G_n].
    \label{eq:asymptotic-counting-telescope}
\end{align}
Since $V_{\rho_Q}\ge0$ and $V_{\sigma_n}\le1+(K-1)u$,
\[
    \E[\rho_Q-\sigma_n\mid\mathcal G_n]
    \le\frac{1+(K-1)u}{p}.
\]
Letting $Q\to\infty$ and applying monotone convergence proves Equation~\eqref{eq:asymptotic-counting-conclusion}.
\end{proof}

On $\{\hat\mu_{1,n}\le b\}$, set $y:=b-\hat\mu_{1,n}$. Apply Lemma~\ref{lem:asymptotic-counting-interval} with $u=u_n(y)$ and $p=p_n(y,B_{\mathrm{dev},-1})$. For fixed $n$, Equation~\eqref{eq:asymptotic-count-threshold} shows that $u_n(y)$ is bounded by a constant times $1+B_{\mathrm{dev},-1}+n y^2+\log(1+ny^2)$; hence it is integrable by Equation~\eqref{eq:asymptotic-deviation-moments} and the Gaussian law of $\hat\mu_{1,n}$. Since $1/p\le2+U_n(y,B_{\mathrm{dev},-1})/E_n(y)$, we obtain, on $\{\hat\mu_{1,n}\le b\}$,
\begin{equation}
    \E[\sigma_{n+1}-\sigma_n\mid\mathcal G_n]
    \le
    C\{1+Ku_n(y)\}
    \left\{1+\frac{U_n(y,B_{\mathrm{dev},-1})}{E_n(y)}\right\}.
    \label{eq:asymptotic-underestimated-interval-bound}
\end{equation}

For use in the logarithmic part of the proof, we also need a bound when $\hat\mu_{1,n}>b$. The next lemma supplies it without depending on $n$ or on the value of $\hat\mu_{1,n}$.

\begin{lemma}[Interval length when $\hat\mu_{1,n}>b$]
\label{lem:asymptotic-normal-interval}
There is a constant $C_b$ such that
\begin{equation}
    \ind\{\hat\mu_{1,n}>b\}
    \E[\sigma_{n+1}-\sigma_n\mid\mathcal G_n]
    \le C_b(1+B_{\mathrm{dev}})^2.
    \label{eq:asymptotic-normal-interval}
\end{equation}
\end{lemma}

\begin{proof}
Let
\[
    a_b:=\frac12\left(b+\max_{j>1}\mu_j\right),
    \qquad
    c_b:=b-a_b=\frac{d_b}{2}>0.
\]
If $\hat\mu_{1,n}>b$, Jensen's inequality gives
\begin{equation}
    \E[(\theta_1-a_b)_+\mid\widetilde{\mathcal F}_{t-1}]
    \ge(\hat\mu_{1,n}-a_b)_+
    \ge c_b.
    \label{eq:asymptotic-normal-best-ei}
\end{equation}

Choose
\begin{equation}
    u_b(B_{\mathrm{dev}})
    :=\left\lceil C_b(1+B_{\mathrm{dev}})\right\rceil,
    \label{eq:asymptotic-normal-count-threshold}
\end{equation}
where $C_b$ is increased below when necessary. For a suboptimal arm with $N_j(t)\ge u_b(B_{\mathrm{dev}})$, the calculation in Equations~\eqref{eq:asymptotic-log-sqrt}-- \eqref{eq:asymptotic-suboptimal-below-separator}, with $a_b$ in place of $b$, gives
\begin{equation}
    \hat\mu_j(t)\le a_b-\frac{c_b}{2},
    \qquad
    \Prob(\theta_j>a_b\mid\widetilde{\mathcal F}_{t-1})
    \le\frac1{4r}.
    \label{eq:asymptotic-normal-tail}
\end{equation}
The exponential bound in Lemma~\ref{lem:asymptotic-pairwise-suboptimal-ei}, with the larger random variable $B_{\mathrm{dev}}$ in place of $B_{\mathrm{dev},-1}$, also gives
\begin{equation}
    \min\{G_{ij}(t),G_{ji}(t)\}
    \le\frac{2^{-r-3}c_b}{K}
    \label{eq:asymptotic-normal-small-pairwise-ei}
\end{equation}
for distinct suboptimal arms whose counts are both at least $u_b(B_{\mathrm{dev}})$. To see both claims explicitly, set $c_0:=c_b^2/8$. Their left-hand sides are bounded by
\[
    \exp\{-c_0u_b(B_{\mathrm{dev}})\}
    \quad\text{and}\quad
    A_\gamma\exp\{C_\gamma B_{\mathrm{dev}}
      -\gamma_-^2u_b(B_{\mathrm{dev}})/16\},
\]
respectively. Taking $C_b$ so that the coefficients of $B_{\mathrm{dev}}$ in both exponents are negative and their values at $B_{\mathrm{dev}}=0$ are below the displayed constants proves Equations~\eqref{eq:asymptotic-normal-tail} and \eqref{eq:asymptotic-normal-small-pairwise-ei}.

Define
\[
    \mathcal S_b(t)
    :=\{j>1:N_j(t)\ge u_b(B_{\mathrm{dev}})\}.
\]
If $\pi_t(\mathcal S_b(t))\ge1/2$, the calculation in Lemma~\ref{lem:asymptotic-best-ei-lower}, with $a_b$ in place of $b$, gives
\begin{equation}
    \bar s_{t,1}
    \ge2^{-r}\frac34c_b.
    \label{eq:asymptotic-normal-best-ei-after-maximum}
\end{equation}
All posterior means are bounded by
\[
    \max_i|\mu_i|+\sqrt{B_{\mathrm{dev}}}+\sqrt\beta,
\]
so Equation~\eqref{eq:asymptotic-posterior-absolute-moment} gives
\begin{equation}
    \max_{i,j}G_{ij}(t)
    \le C\sqrt{1+B_{\mathrm{dev}}}.
    \label{eq:asymptotic-normal-all-pairwise-bound}
\end{equation}
If $\mathcal S_b(t)$ contains an arm with positive sampling probability, choose one, denoted by $\ell_t$, whose empirical mean is the smallest among such arms. For every $j\in\mathcal S_b(t)$ with $\pi_{t,j}>0$, $G_{\ell_tj}(t)=\min\{G_{\ell_tj}(t),G_{j\ell_t}(t)\}$; terms with $\pi_{t,j}=0$ vanish. The EI balance condition and Equations~\eqref{eq:asymptotic-normal-small-pairwise-ei}, \eqref{eq:asymptotic-normal-best-ei-after-maximum}, and \eqref{eq:asymptotic-normal-all-pairwise-bound} therefore give
\begin{align}
    3\cdot2^{-r-2}c_b
    &\le \bar s_{t,1}
    \le \bar s_{t,\ell_t}\notag\\
    &\le \frac{2^{-r-3}c_b}{K}
      +C\sqrt{1+B_{\mathrm{dev}}}
       \left(
         \pi_{t,1}
         +\sum_{\substack{j>1\\N_j(t)<u_b(B_{\mathrm{dev}})}}
            \pi_{t,j}
       \right).
    \label{eq:asymptotic-normal-ei-calculation}
\end{align}
The difference between the first term and $2^{-r-3}c_b/K$ is a positive constant. Dividing Equation~\eqref{eq:asymptotic-normal-ei-calculation} by $C\sqrt{1+B_{\mathrm{dev}}}$ yields
\begin{equation}
    \pi_{t,1}
    +\sum_{\substack{j>1\\N_j(t)<u_b(B_{\mathrm{dev}})}}\pi_{t,j}
    \ge\frac{c}{\sqrt{1+B_{\mathrm{dev}}}}
    \label{eq:asymptotic-normal-progress-probability}
\end{equation}
for a deterministic $c>0$. If $\pi_t(\mathcal S_b(t))<1/2$, the left-hand side is larger than $1/2$. If the set contains no arm with positive sampling probability, the left-hand side equals one. Thus the same displayed lower bound holds in every case after decreasing $c$.

Lemma~\ref{lem:asymptotic-counting-interval}, applied with $u=u_b(B_{\mathrm{dev}})$ and $p=c/\sqrt{1+B_{\mathrm{dev}}}$, now gives the following bound. The required integrability of $u_b(B_{\mathrm{dev}})$ follows from Equation~\eqref{eq:asymptotic-deviation-moments}.
\begin{align*}
    \ind\{\hat\mu_{1,n}>b\}
    \E[\sigma_{n+1}-\sigma_n\mid\mathcal G_n]
    &\le
    C\{1+Ku_b(B_{\mathrm{dev}})\}
      \sqrt{1+B_{\mathrm{dev}}}\\
    &\le C_b(1+B_{\mathrm{dev}})^{3/2}\\
    &\le C_b(1+B_{\mathrm{dev}})^2,
\end{align*}
which proves the lemma.
\end{proof}

We can now sum the intervals on which arm $1$ is underestimated.

\begin{proposition}[Finite expected number of rounds with
$\hat\mu_1(t)\le b$]
\label{prop:asymptotic-finite-underestimation}
For every $b$ satisfying Equation~\eqref{eq:asymptotic-separator},
\begin{equation}
    \E_\mu\!\left[
      \sum_{t=K+1}^{\infty}
      \ind\{\hat\mu_1(t)\le b\}
    \right]<\infty.
    \label{eq:asymptotic-finite-underestimation}
\end{equation}
\end{proposition}

\begin{proof}
First, every $\sigma_n$ is finite almost surely. Indeed, $\sigma_1=K+1$. Conditional on $\mathcal G_n$, either $\hat\mu_{1,n}\le b$, in which case Equation~\eqref{eq:asymptotic-underestimated-interval-bound} is finite, or $\hat\mu_{1,n}>b$, in which case Equation~\eqref{eq:asymptotic-normal-interval} is finite. A nonnegative random variable with finite conditional expectation is finite almost surely. Induction on $n$ therefore proves
\begin{equation}
    \Prob(\sigma_n<\infty)=1,
    \qquad n\ge1.
    \label{eq:asymptotic-all-count-times-finite}
\end{equation}

On $\{\hat\mu_{1,n}\le b\}$, set $y:=b-\hat\mu_{1,n}$ and use Equations~\eqref{eq:asymptotic-count-threshold}, \eqref{eq:asymptotic-threshold-ei-lower}, and \eqref{eq:asymptotic-all-pairwise-bound} in Equation~\eqref{eq:asymptotic-underestimated-interval-bound}. The lower bound on $E_n(y)$ first gives
\begin{equation}
    \log\frac{8K}{E_n(y)}
    \le C+\frac12\log n+\log(1+ny^2)+\frac{ny^2}{2}.
    \label{eq:asymptotic-log-inverse-threshold-ei}
\end{equation}
Using $\log n\le n$ and $\log(1+ny^2)\le ny^2$, we obtain the explicit bound
\begin{align}
    &\ind\{\hat\mu_{1,n}\le b\}
    \E[\sigma_{n+1}-\sigma_n\mid\mathcal G_n]
    \notag\\
    &\quad\le
    C(1+B_{\mathrm{dev},-1})^{3/2}n^{5/2}(1+y)^5
      \exp\!\left(\frac{ny^2}{2}\right).
    \label{eq:asymptotic-explicit-interval-bound}
\end{align}
For completeness, the powers in this display follow from
\begin{align*}
    u_n(y)
    &\le
    C\{n+1+B_{\mathrm{dev},-1}+\log n+\log(1+ny^2)+ny^2\}\\
    &\le C(1+B_{\mathrm{dev},-1})n(1+y^2),\\
    E_n(y)^{-1}
    &\le C\sqrt n(1+ny^2)e^{ny^2/2}
    \le Cn^{3/2}(1+y^2)e^{ny^2/2},\\
    U_n(y,B_{\mathrm{dev},-1})
    &\le C(1+B_{\mathrm{dev},-1})^{1/2}(1+y),
\end{align*}
and $(1+y)(1+y^2)^2\le(1+y)^5$.

Define
\begin{equation}
    \mathcal C_n
    :=
    \sigma\!\left(
      \hat\mu_{1,n},
      X_{j,m}:j>1,\ m\ge1
    \right).
    \label{eq:asymptotic-conditioning-field}
\end{equation}
This sigma-field is contained in $\mathcal G_n$ because all $\sigma_n$ are finite and $\widetilde{\mathcal F}$ contains all reward sequences. Taking conditional expectation of Equation~\eqref{eq:asymptotic-explicit-interval-bound} down to $\mathcal C_n$ preserves the same right-hand side.

The variable $B_{\mathrm{dev},-1}$ is independent of $\hat\mu_{1,n}\sim\N(\mu_1,1/n)$. Therefore
\begin{align}
    &\E\!\left[
      \ind\{\hat\mu_{1,n}\le b\}
      (\sigma_{n+1}-\sigma_n)
    \right]
    \notag\\
    &\quad\le
    C\E[(1+B_{\mathrm{dev},-1})^{3/2}]n^{5/2}
    \int_0^\infty
      (1+y)^5e^{ny^2/2}
      \sqrt{\frac n{2\pi}}
      e^{-n(\varepsilon_b+y)^2/2}\,\dd y\\
    &\quad=
    C\E[(1+B_{\mathrm{dev},-1})^{3/2}]n^3
    e^{-n\varepsilon_b^2/2}
    \int_0^\infty(1+y)^5e^{-n\varepsilon_by}\,\dd y.
    \label{eq:asymptotic-density-cancellation}
\end{align}
The equality uses the exact cancellation
\begin{equation}
    -\frac n2(\varepsilon_b+y)^2+\frac{ny^2}{2}
    =-\frac{n\varepsilon_b^2}{2}-n\varepsilon_by.
    \label{eq:asymptotic-exponent-cancellation}
\end{equation}
Expanding $(1+y)^5$ and using
\begin{equation}
    \int_0^\infty y^q e^{-n\varepsilon_by}\,\dd y
    =\frac{q!}{(n\varepsilon_b)^{q+1}},
    \qquad q=0,\ldots,5,
    \label{eq:asymptotic-gamma-integrals}
\end{equation}
gives
\begin{equation}
    \int_0^\infty(1+y)^5e^{-n\varepsilon_by}\,\dd y
    \le\frac{C_b}{n}.
    \label{eq:asymptotic-integral-bound}
\end{equation}
Equations~\eqref{eq:asymptotic-deviation-moments}, \eqref{eq:asymptotic-density-cancellation}, and \eqref{eq:asymptotic-integral-bound} imply
\begin{equation}
    \E\!\left[
      \ind\{\hat\mu_{1,n}\le b\}
      (\sigma_{n+1}-\sigma_n)
    \right]
    \le C_b n^2e^{-n\varepsilon_b^2/2}.
    \label{eq:asymptotic-one-count-sum-bound}
\end{equation}
The right-hand side is summable in $n$. Finally, by Equation~\eqref{eq:asymptotic-count-interval},
\begin{equation}
    \sum_{t=K+1}^{\infty}\ind\{\hat\mu_1(t)\le b\}
    =
    \sum_{n=1}^{\infty}
      \ind\{\hat\mu_{1,n}\le b\}
      (\sigma_{n+1}-\sigma_n)
    \quad\text{almost surely}.
    \label{eq:asymptotic-reindex-underestimation}
\end{equation}
Taking expectations and summing Equation~\eqref{eq:asymptotic-one-count-sum-bound} proves the proposition.
\end{proof}

\subsection{Asymptotic upper bound for a suboptimal arm}

Fix a suboptimal arm $j>1$, and choose
\begin{equation}
    0<\varepsilon<\min\!\left\{1,\frac{\Delta_{\min}}5\right\},
    \qquad
    b_\varepsilon:=\mu_1-\varepsilon,
    \qquad
    d_{j,\varepsilon}
    :=D(\mu_j+\varepsilon\Vert\mu_1-\varepsilon)
    =\frac{(\Delta_j-2\varepsilon)^2}{2}.
    \label{eq:asymptotic-saturation-parameters}
\end{equation}
Fix $\eta>0$ and choose $\kappa>1$ so that
\begin{equation}
    \rho_\kappa:=\frac{\kappa}{\kappa+1},
    \qquad
    (1+\eta)\rho_\kappa>1.
    \label{eq:asymptotic-kappa-choice}
\end{equation}
Define
\begin{equation}
    m_T
    :=
    \left\lceil
      \frac{r(1+\eta)\log T}{d_{j,\varepsilon}}
    \right\rceil.
    \label{eq:asymptotic-saturation-count}
\end{equation}

On a round satisfying
\begin{equation}
    \hat\mu_1(t)>b_\varepsilon,
    \qquad
    \hat\mu_j(t)\le\mu_j+\varepsilon,
    \label{eq:asymptotic-good-empirical-means}
\end{equation}
the empirical gap is larger than $\Delta_j-2\varepsilon$. The pairwise saturation inequality in Lemma~\ref{lem:generalM-saturation-probability} therefore gives
\begin{equation}
    \pi_{t,j}
    \le
    A_{j,M}
    \exp\!\left\{
      -\frac{d_{j,\varepsilon}}r
      \frac{N_1(t)N_j(t)}{N_1(t)+N_j(t)}
    \right\}.
    \label{eq:asymptotic-pi-upper}
\end{equation}
Indeed, the exponent $f_t^\star/r$ in Equation~\eqref{eq:generalM-pi-j} is at least the exponent displayed above. The constant $A_{j,M}$ is independent of $t$ and $T$.

Lemma~\ref{lem_armj_overestimated} and its proof show
\begin{equation}
    \E\!\left[
      \sum_{t=K+1}^T
      \ind\{I(t)=j,\ \hat\mu_j(t)>\mu_j+\varepsilon\}
    \right]
    \le\frac2{\varepsilon^2}.
    \label{eq:asymptotic-overestimation-count}
\end{equation}

There are at most $m_T$ pulls of arm $j$ on rounds with $N_j(t)<m_T$. It remains to consider rounds satisfying Equations~\eqref{eq:asymptotic-good-empirical-means} and $N_j(t)\ge m_T$. First suppose
\begin{equation}
    N_1(t)\ge\kappa N_j(t).
    \label{eq:asymptotic-best-count-large}
\end{equation}
Then
\begin{equation}
    \frac{N_1(t)N_j(t)}{N_1(t)+N_j(t)}
    \ge\rho_\kappa N_j(t)
    \ge\rho_\kappa m_T.
\end{equation}
Equations~\eqref{eq:asymptotic-pi-upper} and \eqref{eq:asymptotic-saturation-count} imply the following one-round bound. All conditions in the indicator are $\widetilde{\mathcal F}_{t-1}$-measurable, so
\begin{align}
    &\E\!\left[
      \ind\{\hat\mu_1(t)>b_\varepsilon,
      \hat\mu_j(t)\le\mu_j+\varepsilon,
      N_j(t)\ge m_T,
      N_1(t)\ge\kappa N_j(t)\}
      \ind\{I(t)=j\}
    \right]\notag\\
    &=\E\!\left[
      \ind\{\hat\mu_1(t)>b_\varepsilon,
      \hat\mu_j(t)\le\mu_j+\varepsilon,
      N_j(t)\ge m_T,
      N_1(t)\ge\kappa N_j(t)\}
      \pi_{t,j}
    \right]\notag\\
    &\le A_{j,M}
      \exp\!\left(-\frac{d_{j,\varepsilon}\rho_\kappa m_T}{r}\right).
    \label{eq:asymptotic-large-best-count-one-round}
\end{align}
Summing this inequality over at most $T$ rounds gives
\begin{align}
    &\E\!\left[
      \sum_{t=K+1}^T
      \ind\{I(t)=j,\
      \hat\mu_1(t)>b_\varepsilon,\
      \hat\mu_j(t)\le\mu_j+\varepsilon,\
      N_j(t)\ge m_T,\
      N_1(t)\ge\kappa N_j(t)\}
    \right]
    \notag\\
    &\quad\le
    A_{j,M}T
    \exp\!\left(-\frac{d_{j,\varepsilon}\rho_\kappa m_T}{r}\right)\\
    &\quad\le
    A_{j,M}T^{1-(1+\eta)\rho_\kappa}
    =o(1).
    \label{eq:asymptotic-large-best-count-sum}
\end{align}

Now suppose $N_1(t)<\kappa N_j(t)$ and group the rounds by $n=N_1(t)$. On such a round, the harmonic count in Equation~\eqref{eq:asymptotic-pi-upper} satisfies
\begin{equation}
    \frac{nN_j(t)}{n+N_j(t)}
    \ge
    \begin{cases}
      n/2,&1\le n\le m_T,\\
      m_T/2,&m_T<n\le\kappa m_T,\\
      n/(\kappa+1),&n>\kappa m_T.
    \end{cases}
    \label{eq:asymptotic-three-harmonic-bounds}
\end{equation}
The first line uses $N_j(t)\ge m_T\ge n$, the second uses $N_j(t)\ge m_T$ and $n>m_T$, and the third uses $N_j(t)>n/\kappa$.

Accordingly, define the deterministic bound
\begin{equation}
    a_{n,T}
    :=A_{j,M}
    \begin{cases}
      \exp\{-d_{j,\varepsilon}n/(2r)\},
        &1\le n\le m_T,\\
      \exp\{-d_{j,\varepsilon}m_T/(2r)\},
        &m_T<n\le\kappa m_T,\\
      \exp\{-d_{j,\varepsilon}n/[r(\kappa+1)]\},
        &n>\kappa m_T.
    \end{cases}
    \label{eq:asymptotic-three-probability-bounds}
\end{equation}
Equation~\eqref{eq:asymptotic-pi-upper} gives $\pi_{t,j}\le a_{n,T}$ on every round in the $n$th interval under consideration.

For each $n$, let $A_{n,t}$ be the following pre-decision event:
\begin{equation}
\begin{split}
    A_{n,t}:=\{&K+1\le t\le T,\ N_1(t)=n,\
      \hat\mu_{1,n}>b_\varepsilon,\\
      &\hat\mu_j(t)\le\mu_j+\varepsilon,\
      N_j(t)\ge m_T,\ n<\kappa N_j(t)\}.
\end{split}
    \label{eq:asymptotic-predecision-event}
\end{equation}
Every condition in $A_{n,t}$ is $\widetilde{\mathcal F}_{t-1}$-measurable. Therefore Equations~\eqref{eq:asymptotic-action-law}, \eqref{eq:asymptotic-three-probability-bounds}, and the tower property give
\begin{align}
    &\E\!\left[
      \sum_{t=K+1}^T\ind\{A_{n,t}\}\ind\{I(t)=j\}
      \middle|\mathcal G_n
    \right]\\
    &=
    \E\!\left[
      \sum_{t=K+1}^T\ind\{A_{n,t}\}\pi_{t,j}
      \middle|\mathcal G_n
    \right]\\
    &\le
    a_{n,T}\,
    \ind\{\hat\mu_{1,n}>b_\varepsilon\}
    \E[\sigma_{n+1}-\sigma_n\mid\mathcal G_n]\\
    &\le
    a_{n,T}C_{b_\varepsilon}(1+B_{\mathrm{dev}})^2,
    \label{eq:asymptotic-one-interval-actions}
\end{align}
where the last line is Lemma~\ref{lem:asymptotic-normal-interval}. Taking expectations and using Equation~\eqref{eq:asymptotic-deviation-moments}, the total contribution from $N_1(t)<\kappa N_j(t)$ is at most a constant times $\sum_{n\ge1}a_{n,T}$.

We now evaluate this sum. First,
\begin{equation}
    \sum_{n=1}^{m_T}e^{-d_{j,\varepsilon}n/(2r)}
    \le
    \frac1{e^{d_{j,\varepsilon}/(2r)}-1}
    =O_{j,\varepsilon,M}(1).
    \label{eq:asymptotic-first-geometric-sum}
\end{equation}
Second, Equation~\eqref{eq:asymptotic-saturation-count} gives
\begin{align}
    \sum_{m_T<n\le\kappa m_T}
      e^{-d_{j,\varepsilon}m_T/(2r)}
    &\le
    \kappa m_T e^{-d_{j,\varepsilon}m_T/(2r)}\\
    &\le C\log T\,T^{-(1+\eta)/2}
    =o(1).
    \label{eq:asymptotic-middle-geometric-sum}
\end{align}
Third,
\begin{align}
    \sum_{n>\kappa m_T}
      e^{-d_{j,\varepsilon}n/[r(\kappa+1)]}
    &\le
    \frac{
      e^{-d_{j,\varepsilon}\kappa m_T/[r(\kappa+1)]}
    }{
      1-e^{-d_{j,\varepsilon}/[r(\kappa+1)]}
    }\\
    &\le C T^{-(1+\eta)\rho_\kappa}
    =o(1).
    \label{eq:asymptotic-last-geometric-sum}
\end{align}
Thus the contribution in Equation~\eqref{eq:asymptotic-one-interval-actions}, summed over $n$, is $O(1)$.

Combine Proposition~\ref{prop:asymptotic-finite-underestimation}, Equation~\eqref{eq:asymptotic-overestimation-count}, the at most $m_T$ pulls before $N_j(t)$ reaches $m_T$, Equation~\eqref{eq:asymptotic-large-best-count-sum}, and Equations~\eqref{eq:asymptotic-first-geometric-sum}-- \eqref{eq:asymptotic-last-geometric-sum}. Including the initial pull, we obtain
\begin{equation}
    \E_\mu[N_j(T+1)]
    \le
    m_T+O_{\mu,K,M,\varepsilon,\eta}(1)+o(1).
    \label{eq:asymptotic-arm-count-finite-T}
\end{equation}
Divide by $\log T$ and use Equation~\eqref{eq:asymptotic-saturation-count}:
\begin{equation}
    \limsup_{T\to\infty}
    \frac{\E_\mu[N_j(T+1)]}{\log T}
    \le\frac{r(1+\eta)}{d_{j,\varepsilon}}.
    \label{eq:asymptotic-arm-count-epsilon-eta}
\end{equation}
First let $\eta\downarrow0$ and then let $\varepsilon\downarrow0$. Since $d_{j,\varepsilon}\to D(\mu_j\Vert\mu_1)$, this proves
\begin{equation}
    \limsup_{T\to\infty}
    \frac{\E_\mu[N_j(T+1)]}{\log T}
    \le\frac{M-1}{D(\mu_j\Vert\mu_1)}.
    \label{eq:asymptotic-arm-count-upper-proof}
\end{equation}
Multiplying by $\Delta_j$, summing over the finitely many suboptimal arms, and using $\Reg(T)=\sum_{j=2}^K\Delta_jN_j(T+1)$ proves Equation~\eqref{eq:asymptotic-regret-upper}.

\subsection{Matching lower bound for \(M=2\)}

Setting \(M=2\) in Equation~\eqref{eq:asymptotic-count-upper} gives, for every \(j>1\),
\[
    \limsup_{T\to\infty}
    \frac{\E_\mu[N_j(T+1)]}{\log T}
    \le
    \frac{1}{D(\mu_j\Vert\mu_1)}.
\]

The preceding logarithmic upper bound holds for every fixed pairwise-distinct unit-variance Gaussian instance. 
Hence ReMax with \(M=2\) is uniformly good: namely,
\[
    \E_\mu[\Reg(T)]
    =
    O_\mu(\log T)
    =
    o(T^a)
    \qquad
    \text{for every }a>0
\]
and for any distinct\footnote{Note that excluding the case of non-distinct subclass of arms does not change the discussion of \cite{lairobbins1985}.} set of arms.
It is well-known that \citep{lairobbins1985}, for any uniformly good algorithm, we have the following lower bound
\[
    \liminf_{T\to\infty}
    \frac{\E_\mu[N_j(T+1)]}{\log T}
    \ge
    \frac{1}{D(\mu_j\Vert\mu_1)},
    \qquad j>1.
\]
Combining the upper and lower bounds yields
\[
    \lim_{T\to\infty}
    \frac{\E_\mu[N_j(T+1)]}{\log T}
    =
    \frac{2}{\Delta_j^2},
    \qquad j>1.
\]

Since
\[
    \E_\mu[\Reg(T)]
    =
    \sum_{j=2}^K
    \Delta_j\,\E_\mu[N_j(T+1)]
\]
and the number of arms is finite, the preceding limit gives
\[
    \lim_{T\to\infty}
    \frac{\E_\mu[\Reg(T)]}{\log T}
    =
    2\sum_{j=2}^K\frac{1}{\Delta_j}.
\]
This completes the proof of Theorem~\ref{thm:asymptotic-regret}.

\section{ReMaxGrad implementation}
\label{app:remaxgrad-implementation}

ReMaxGrad~\citep{nishimori2026emergence} approximates the maximization in Equation~\eqref{eq:remax-policy} by stochastic gradient ascent in logit space. The outer loop corresponds to the bandit interaction rounds indexed by $t$. At each round $t$, with the posterior $\Pi_{t-1}$ fixed, ReMaxGrad runs an inner optimization loop over the sampling distribution $\pi$. At inner optimization step $\ell$, write $\pi=\operatorname{softmax}(\bm z)$ and draw $S$ independent posterior samples $\bm\theta_{\ell}^{(1)},\ldots,\bm\theta_{\ell}^{(S)}\overset{\mathrm{iid}}{\sim}\Pi_{t-1}$. The corresponding Monte Carlo objective is
\begin{equation}
    \widehat J_{t,\ell}^{(M)}(\pi)
    :=
    \frac{1}{S}\sum_{s=1}^{S}
    \E_{A_{1:M}\sim\pi}
    \left[
        \max_{m\in[M]}\theta_{\ell,A_m}^{(s)}
    \right].
    \label{eq:remaxgrad-sampled-objective}
\end{equation}
For each posterior sample, the gradient of the inner expectation over the virtual actions is computed exactly; Monte Carlo sampling is used only for the posterior expectation. Let $\widehat{\bm g}_{t,\ell}:=\nabla_{\pi}\widehat J_{t,\ell}^{(M)}(\pi)$. The gradient with respect to the logits and the sampled simplex KKT gap are
\begin{align}
    \widehat{\bm h}_{t,\ell}
    &:=
    \nabla_{\bm z}\widehat J_{t,\ell}^{(M)}(\pi)
    =
    \pi\odot
    \left(
        \widehat{\bm g}_{t,\ell}
        -\langle\pi,\widehat{\bm g}_{t,\ell}\rangle\mathbf 1
    \right),
    \label{eq:remaxgrad-logit-gradient}\\
    \widehat{\mathcal G}_{t,\ell}(\pi)
    &:=
    \max_{i\in[K]}\widehat g_{t,\ell,i}
    -\langle\pi,\widehat{\bm g}_{t,\ell}\rangle.
    \label{eq:remaxgrad-kkt-gap}
\end{align}
We use the sampled gap in Equation~\eqref{eq:remaxgrad-kkt-gap} as the early-stopping criterion. Algorithm~\ref{alg:remaxgrad} summarizes the complete procedure.

\begin{algorithm}[htbp]
\caption{ReMaxGrad}
\label{alg:remaxgrad}
\begin{algorithmic}[1]
\Require Retry budget $M$, horizon $T$, and optimizer parameters $S,L,\eta,\xi$
\State Pull every arm once, update the posterior, and initialize $\bm z\gets\bm 0$
\For{$t=K+1,\ldots,T$}
    \State Reset the Adam moments $\bm m\gets\bm 0$ and $\bm v\gets\bm 0$
    \For{$\ell=1,\ldots,L$}
        \State $\pi\gets\operatorname{softmax}(\bm z)$
        \State Draw $S$ fresh posterior samples from $\Pi_{t-1}$
        \State Compute $\widehat{\bm g}_{t,\ell}$ and $\widehat{\mathcal G}_{t,\ell}(\pi)$
        \If{$\widehat{\mathcal G}_{t,\ell}(\pi)\leq\xi$}
            \State \textbf{break}
        \EndIf
        \State Compute $\widehat{\bm h}_{t,\ell}$ from Equation~\eqref{eq:remaxgrad-logit-gradient}
        \State Update $(\bm z,\bm m,\bm v)$ by Adam ascent with learning rate $\eta$
    \EndFor
    \State Set $\pi_t\gets\operatorname{softmax}(\bm z)$ and sample $I(t)\sim\pi_t$
    \State Observe the reward and update the posterior to obtain $\Pi_t$
\EndFor
\end{algorithmic}
\end{algorithm}

\begin{table}[htbp]
    \centering
    \caption{ReMaxGrad settings used in the main experiments.}
    \label{tab:remaxgrad-settings}
    \small
    \begin{tabular}{lc}
        \toprule
        Parameter & Setting \\
        \midrule
        Retry budget $M$ & $\{2,3,4\}$ \\
        Posterior samples per gradient step $S$ & $50$ \\
        Maximum Adam steps per round $L$ & $50$ \\
        Learning rate $\eta$ & $0.02$ \\
        Adam $(\beta_1,\beta_2,\epsilon_{\mathrm{Adam}})$ & $(0.9,0.999,10^{-8})$ \\
        Sampled KKT tolerance $\xi$ & $10^{-7}$ \\
        Batch size (fixed-instance experiment) & $100$ \\
        Batch size (random-instance experiment) & $1{,}000$ \\
        \bottomrule
    \end{tabular}
\end{table}

The first adaptive round starts from the uniform policy. In later rounds, the logits are warm-started from the preceding round, while the Adam~\citep{kingma2014adam} moments are reset. Algorithm~\ref{alg:remaxgrad} is written for one trajectory; in the vectorized implementation, early stopping occurs only when the largest sampled gap in the batch is below $\xi$. Posterior draws and the sampled-gradient kernel use float32, whereas the posterior state, logits, Adam arithmetic, action probabilities, regret accumulation, and diagnostics use float64. Since the inner optimization is capped at $L=50$, the implementation is a fixed-budget approximation to the exact ReMax optimizer. 

\paragraph{Comparison with exact ReMax for \texorpdfstring{$M=2$}{M=2}}
\label{app:remaxgrad-exact-m2}

For $M=2$, define $Q_{t,ij}:=\E[\max\{\theta_i,\theta_j\}]$ for $i\neq j$ and $Q_{t,ii}:=\E[\theta_i]=\hat\mu_i(t)$. The ReMax objective is then the concave quadratic program
\begin{equation}
    J_t^{(2)}(\pi)=\pi^\top Q_t\pi,
    \qquad
    \pi\in\simplex.
    \label{eq:m2-quadratic-objective}
\end{equation}
Let $\mathcal A:=\supp(\pi_t^\star)$. The KKT conditions give $(Q_t\pi_t^\star)_i=\lambda$ for $i\in\mathcal A$ and $(Q_t\pi_t^\star)_i\leq\lambda$ for $i\notin\mathcal A$. Hence, whenever $Q_{t,\mathcal A\mathcal A}$ is invertible,
\begin{equation}
    \pi_{t,\mathcal A}^\star
    =
    \frac{Q_{t,\mathcal A\mathcal A}^{-1}\mathbf 1}
    {\mathbf 1^\top Q_{t,\mathcal A\mathcal A}^{-1}\mathbf 1},
    \qquad
    \pi_{t,\mathcal A^c}^\star=\mathbf 0.
    \label{eq:m2-active-set-solution}
\end{equation}
The analytical solver updates the active set until positivity and the inactive-arm inequalities are satisfied. To measure the optimization error of ReMaxGrad against this exact objective, we use the KKT gap
\begin{equation}
    \mathcal G_t^{(2)}(\pi)
    :=
    \max_{i\in[K]}[2Q_t\pi]_i
    -\langle\pi,2Q_t\pi\rangle,
    \label{eq:m2-exact-kkt-gap}
\end{equation}
which is zero at the exact maximizer and upper-bounds $J_t^{(2)}(\pi_t^\star)-J_t^{(2)}(\pi)$ by concavity~\citep{vandenberghe2004convex}.

We compare the analytical $M=2$ optimizer with ReMaxGrad on the fixed instance in Section~\ref{sec:fixed-frequentist}, using ten batches of $100$ trajectories with $T=20{,}000$. ReMaxGrad uses the same optimizer settings and batch size as the main fixed-instance experiment (Table~\ref{tab:remaxgrad-settings}). We evaluate its KKT gap under the exact posterior objective in Equation~\eqref{eq:m2-exact-kkt-gap}, rather than the sampled objective.

\begin{figure}[htbp]
    \centering
    \includegraphics[width=\linewidth]{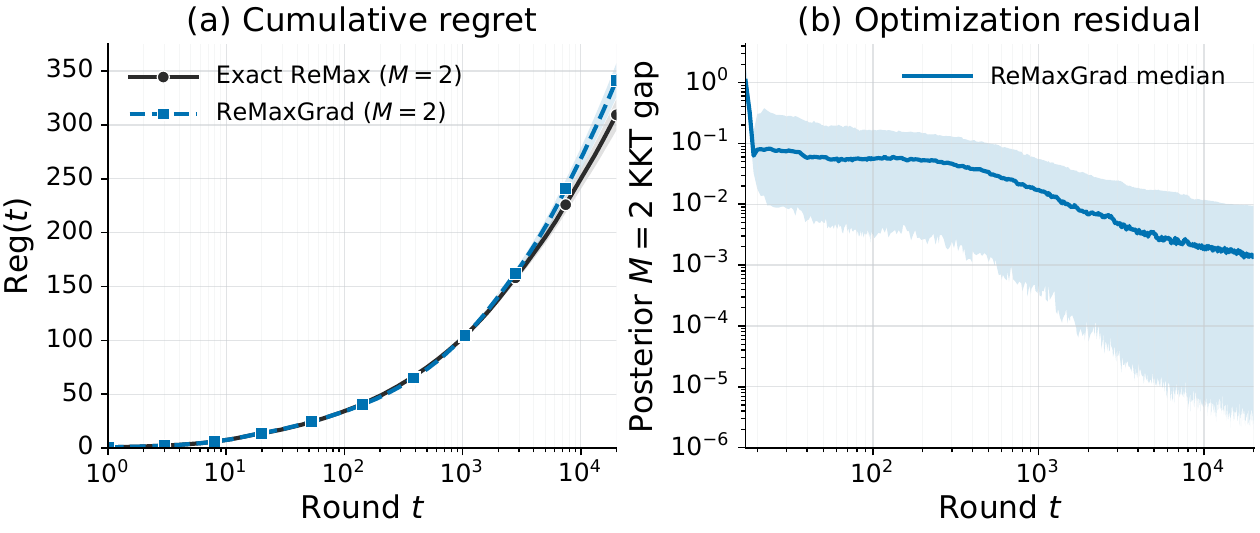}
    \vspace{-2em}
    \caption{Comparison of the analytical $M=2$ optimizer and ReMaxGrad on the fixed instance ($T=20{,}000$, $\xi=10^{-7}$). (a) Mean cumulative regret over $1{,}000$ trajectories, with one-standard-error bands. (b) Median exact-posterior KKT gap of ReMaxGrad, with $5$th--$95$th percentile bands.}
    \label{fig:remaxgrad-exact-m2}
\end{figure}

\paragraph{Results} Figure~\ref{fig:remaxgrad-exact-m2} shows that ReMaxGrad follows the overall regret trend of the analytical optimizer, with a higher mean cumulative regret at $T=20{,}000$ ($341.4$ versus $309.4$). Its median exact-posterior KKT gap generally decreases over time, reaching approximately $1.4\times10^{-3}$ at the final round. In contrast, the exact-posterior KKT gap of the analytical solver remains below $10^{-8}$ on every adaptive round.

\end{document}